%% file: main_camera_ready.tex
\documentclass{article}

\usepackage{microtype}
\usepackage{graphicx}
\usepackage{subcaption}
\usepackage{booktabs}
\usepackage{multirow}
\usepackage{tcolorbox}
\usepackage[table]{xcolor}
\definecolor{baselinebg}{RGB}{255,240,230}
\definecolor{refbg}{RGB}{235,245,255}
\definecolor{reprobg}{RGB}{245,245,245}
\definecolor{default}{RGB}{215,215,215}

\usepackage{amsmath}
\usepackage{amssymb}
\usepackage{mathtools}
\usepackage{amsthm}
\usepackage{listings}

\usepackage{algorithm}
\usepackage{algorithmic}
\usepackage{float}    % enables [H] placement for side-by-side algorithm minipages
\usepackage{caption}  % \captionof{table}{...} for the notation block
\usepackage{wrapfig}  % wraptable for compact tables next to prose
\usepackage{needspace} % prevent section headings from being orphaned at page bottoms
\usepackage{tikz}
\usetikzlibrary{positioning,arrows.meta,fit,calc,backgrounds,shapes.geometric}
\usepackage[main, final]{neurips_2026}

\usepackage[utf8]{inputenc}
\usepackage[T1]{fontenc}
\usepackage{hyperref}
\usepackage{url}
\hypersetup{hidelinks}

\usepackage[capitalize,noabbrev]{cleveref}

\theoremstyle{plain}
\newtheorem{theorem}{Theorem}[section]
\newtheorem{proposition}[theorem]{Proposition}

\newtheorem{corollary}[theorem]{Corollary}
\theoremstyle{definition}
\newtheorem{definition}[theorem]{Definition}
\newtheorem{assumption}[theorem]{Assumption}
\theoremstyle{remark}
\newtheorem{remark}[theorem]{Remark}

\input{command}

\title{Saliency-Aware Multi-Route Thinking:\\ Grounding and Reasoning on Vision-Language Agents}

\author{%
  Mingjia Shi$^{1,*}$ \quad Yinhan He$^{1}$ \quad Yaochen Zhu$^{1}$ \quad
  Cassie Dong$^{2}$ \quad Jundong Li$^{1,\dagger}$ \\[0.35em]
  $^{1}$University of Virginia \qquad $^{2}$Meta \\[0.25em]
  $^{*}$\texttt{nzh3ru@virginia.edu} \quad
  $^{\dagger}$\texttt{jundong@virginia.edu}
}

\begin{document}

\maketitle

\begin{abstract}
\input{sec/0_abs/abs_camera_ready}
\end{abstract}

\section{Introduction}
\label{sec:intro}
\input{sec/1_intro/intro_camera_ready}

\section{Problem Formulation}

\input{sec/2_mtd/sec/prob_modeling_v1}

\section{Methodology}
\input{sec/2_mtd/sec/mtd_only_v2}
\input{sec/2_mtd/mtd_camera_ready}

\section{Empirical Study}

\input{sec/3_exp/exp_camera_ready}

\section{Related Works}
\input{sec/4_rlt/rlt_camera_ready}

\section{Conclusions}
\input{sec/5_sum/sum_camera_ready}

\paragraph{Impact statement.}
Improving visual grounding may make VLM outputs more reliable, but SAP's multi-route inference requires more total model computation than a single short response and can therefore increase energy use and its associated environmental footprint. Parallel execution reduces wall-clock latency but does not eliminate this additional computation, so the impact depends on deployment scale, hardware efficiency, and the electricity source.

\begin{ack}
This work was supported in part by the National Science Foundation (NSF) under Grants IIS-2144209, IIS-2223769, CMMI-2411248, and IIBR-2601942; by the Office of Naval Research (ONR) under Grant N000142412636; and by grant support from Meta. Research was also supported by the NVIDIA Academic Grant Program using DGX Spark.

The authors declare no competing interests.
\end{ack}

\clearpage
\bibliographystyle{plainnat}
\bibliography{example_paper_camera_ready}

%%%%%%%%%%%%%%%%%%%%%%%%%%%%%%%%%%%%%%%%%%%%%%%%%%%%%%%%%%%%%%%%%%%%%%%%%%%%%%%
% APPENDIX
%%%%%%%%%%%%%%%%%%%%%%%%%%%%%%%%%%%%%%%%%%%%%%%%%%%%%%%%%%%%%%%%%%%%%%%%%%%%%%%
\newpage
\appendix
\input{sec/6_appdx/appdx_camera_ready}

%%%%%%%%%%%%%%%%%%%%%%%%%%%%%%%%%%%%%%%%%%%%%%%%%%%%%%%%%%%%%%%%%%%%%%%%%%%%%%%
% NeurIPS Paper Checklist (REQUIRED — must follow references / appendix)
%%%%%%%%%%%%%%%%%%%%%%%%%%%%%%%%%%%%%%%%%%%%%%%%%%%%%%%%%%%%%%%%%%%%%%%%%%%%%%%
\newpage
\input{neurips/checklist_camera_ready}

\end{document}

%% file: command.tex
\newcommand{\tablestyle}[2]{\setlength{\tabcolsep}{#1}\renewcommand{\arraystretch}{#2}\centering\footnotesize}

\definecolor{mdfblue}{RGB}{0,90,200}
\definecolor{mdfred}{RGB}{200,30,30}
\newcommand{\MDF}[1]{#1}
\newcommand{\MDFTODO}[1]{}

%% file: sec/0_abs/abs_camera_ready.tex
Vision-language (VL) agents increasingly combine extended reasoning with external visual tools. Yet long traces can drift from the image as generated text accumulates, while directly inserting tool outputs can turn one incorrect detection into a persistent premise. We ask how inference-time computation can improve reasoning without surrendering control to the tool. We propose \emph{Saliency-Aware Principle Selection} (SAP), a training-free method that evolves compact reasoning principles rather than lengthening a single token trajectory. Each principle generates independent routes from the original image and prompt. A frozen grounding tool scores whether objects cited along these routes are visually supported; agreement, diversity, and uncertainty further control drift, collapse, and overconfidence. The tool supplies selection evidence but never writes the reasoning or answer. Across object-hallucination benchmarks, SAP consistently improves visually grounded prediction under a unified model and decoding budget. Broader evaluations spanning reasoning, knowledge, OCR, documents, diagrams, and real-world imagery show that this gain does not come at the expense of general VL capability. Cross-architecture and model-scale experiments further demonstrate that the principle-selection mechanism transfers beyond a single backbone, while controlled ablations confirm that both multi-route exploration and evolutionary selection are necessary. Because routes run in parallel, SAP also reduces serial latency relative to an equally long chain. Together, these results show that scaling over governed reasoning behaviors can preserve visual evidence more effectively than simply extending the chain.
\noindent Code is available at \url{https://github.com/BDeMo/Saliency-Aware-Principle-Selection-NeurIPS-2026-Official}.

%% file: sec/1_intro/intro_camera_ready.tex
\input{sec/1_intro/fig/fig_intro_merged}
\input{sec/2_mtd/fig/fig_pipeline}

Vision--language (VL) agents increasingly use images when answering visual questions~\citep{bai2025qwen3vl}, following language instructions in embodied navigation~\citep{cheng2024navila}, and calling external visual tools~\citep{zhao2024mitigating,wang2025vgr}. Across these settings, the image is encoded only once, at the beginning of autoregressive reasoning, while textual states are regenerated at every step. As the chain grows, an early visual error, such as mistaking a toy for a task-relevant object, can be repeated and elaborated as if it were grounded in the image. Figure~\ref{fig:intro} shows this failure developing: attention to image tokens falls as a single reasoning chain extends. Consequently, an unsupported textual hypothesis becomes harder to correct and can propagate through later reasoning into the final answer. For an agent, this can produce a wrong answer or an action directed toward an object that is not present.

Long chain-of-thought adds inference compute by extending one textual trajectory, but it does not restore visual evidence once that trajectory becomes text dominated~\citep{lu2019vilbert}. Sampling more trajectories has the same limitation~\citep{wang2022self,yao2023tree}. Grounded reasoning instead inserts boxes, captions, or object lists into the VLM context~\citep{wang2025vgr,wang2026treevgr,zhao2024mitigating}; an incorrect tool output can then affect every later step. In our controlled SAM-augmented LongCoT study, direct insertion reduces POPE absent-object recall from 79.6 to 14.8, and advisory wording recovers only part of the loss (Appendix~\ref{appdx:additional:samaug}). The remaining open question is \emph{when and how an agent should call its tools while remaining able to disagree with them.}

We introduce \emph{Saliency-Aware Principle Selection} (SAP), a training-free search over short reasoning principles (Fig.~\ref{fig:framework}). A principle specifies how the VLM should inspect and verify visual evidence without prescribing a token-level trace. The frozen VLM reads each principle and generates several independent routes from the original image and prompt. Searching over principles therefore changes reasoning behavior while leaving image interpretation and answer generation with the VLM.

A frozen grounding tool evaluates the resulting behavior. For each principle, SAP measures whether objects cited by its routes appear in the grounding output and combines this evidence-validity signal with route agreement, within-principle diversity, and uncertainty. Evolutionary selection retains stronger principles and produces the next population through mutation and recombination. Here, saliency means consistency between cited and grounded objects; it does not model human visual attention or assume that every detected object is task-relevant. The tool affects which principles survive but does not write the routes or answers. SAP changes no model parameters, stores no state across queries, and allows routes within a generation to run in parallel.

We evaluate the targeted behavior on POPE, HallusionBench, and AMBER-d under a unified Qwen3-VL-8B setup. Sixteen additional vision, reasoning, OCR, document, diagram, and real-world benchmarks test whether SAP sacrifices capabilities outside object grounding. We match route-token budgets, include grounding and selection overhead, and vary the search components, backbone architecture, and model scale. SAP improves the three object-hallucination evaluations while remaining competitive on the broader suite; the additional analyses identify the source and inference cost of these gains. Our contributions are:
\begin{enumerate}
    \item We identify two failures in VL inference: long textual trajectories lose visual support, while direct tool insertion can propagate grounding errors through later reasoning.
    \item We develop SAP, which uses grounded object support to select high-level reasoning principles without training the VLM or treating tool output as the answer.
    \item SAP improves three object-hallucination evaluations; broader capability, transfer, ablation, and cost experiments characterize the scope and inference trade-offs.
\end{enumerate}

%% file: sec/1_intro/fig/fig_intro_merged.tex
% Half-width wrap-around figure showing the empirical "text-dominated
% reasoning" phenomenon. Place this \input{...} immediately *before* the
% intro paragraph that discusses the attention decay (so the wrapfigure
% sits beside that paragraph).

\begin{wrapfigure}[12]{r}{0.40\textwidth}
\centering
\includegraphics[width=\linewidth]{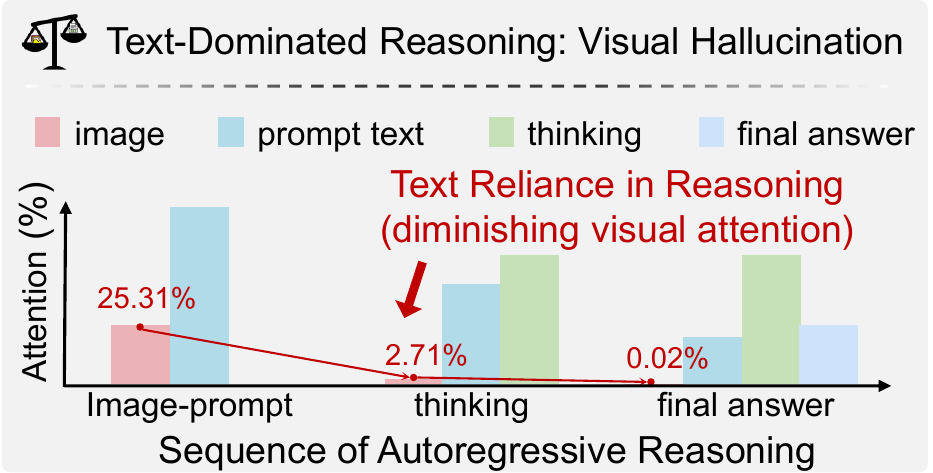}
\caption{Image attention collapses along Qwen3-VL-8B-Thinking on MS-COCO.}
\label{fig:intro}
\end{wrapfigure}

%% file: sec/2_mtd/fig/fig_pipeline.tex
% SAP pipeline / framework figure for the Methodology section.
% Source PDF lives in sec/1_intro/img/ (regenerated from pipeline.pptx).

\begin{figure*}[t]
\centering
\includegraphics[width=\linewidth]{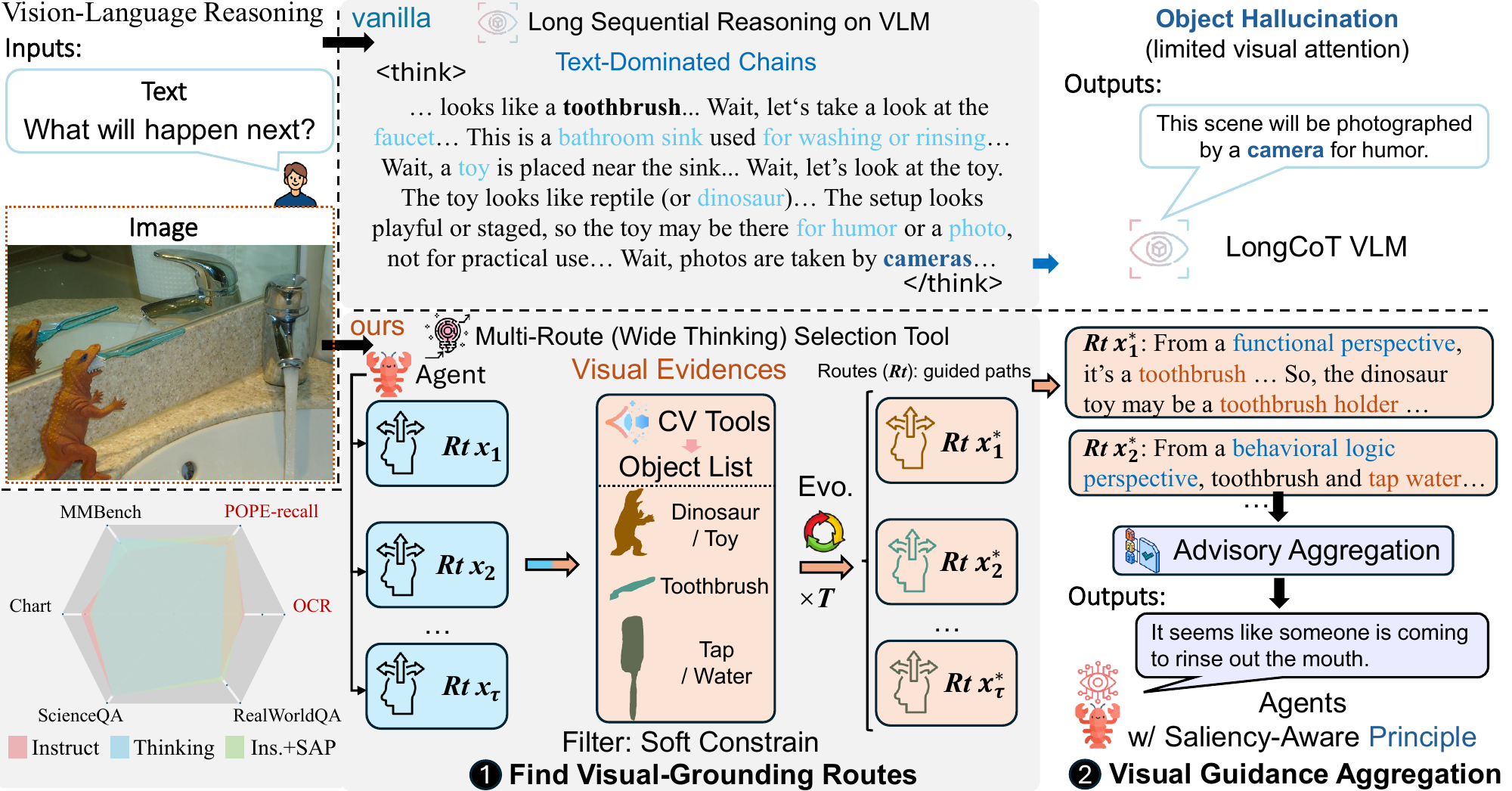}
\caption{The motivation and pipeline of \textbf{Saliency-Aware Principle (SAP)} route selection tool.
Given an image $\mathcal{I}$ and a prompt $p$, a frozen grounding tool
(SAM + captioner) produces a textual grounding summary $g$ that is fed
to the agent only as an \emph{advisory} signal. The frozen language model
$\varphi$ maintains a population of $\mu+\lambda$ reasoning principles
$\{x_i\}$; each principle instantiates $\tau$ short parallel reasoning
routes $\{r_{i,j}\}$. Four ordinal signals---consensus $c$, diversity
$d$, evidence validity $e$, uncertainty $u$---form a discrete fitness
$f_i$, and a $(\mu+\lambda)$ elite-selection scheme keeps the top
principles while sampling $\lambda$ new ones from the elites. Iterating
this loop $T$ times realizes inference-time scaling through
\emph{principle-level evolution}, rather than through extending a single
reasoning chain.}
\label{fig:framework}
\end{figure*}

%% file: sec/2_mtd/sec/prob_modeling_v1.tex
In this paper, we formulate inference-time scaling as an optimization problem over reasoning routes executed during the inference phase.
Given a prompt $p$ and a language model $\varphi$, a reasoning route 

$\rho \in \mathcal{R}$ specifies how inference unfolds across steps---determining the structure and progression of intermediate reasoning---where $\mathcal{R}$ denotes the space of feasible reasoning routes.
The objective of inference-time scaling is to identify a reasoning route that maximizes a task utility $f(\rho; p, \varphi)$, which scores the outputs that route $\rho$ induces under prompt $p$ and language model $\varphi$. Suppressing $(p, \varphi)$ when they are held fixed, we write
\begin{equation}
\label{eq:obj}
\rho^\star = \arg\max_{\rho \in \mathcal{R}} \; f(\rho).
\end{equation}
Here, inference-time scaling does not refer solely to generating a longer chain of reasoning along a single trajectory.
Instead, it refers to allocating additional inference-time computation to \emph{explore multiple alternative reasoning routes} and to select higher-quality ones.
As the inference budget increases, more candidates in $\mathcal{R}$ can be instantiated and evaluated, improving the likelihood of identifying a reasoning route with higher utility.
In this sense, inference-time scaling corresponds to scaling the \emph{search over reasoning behaviors} rather than merely extending a single route.

%% file: sec/2_mtd/sec/mtd_only_v2.tex
To address the two failure modes of Section~\ref{sec:intro}, we propose \emph{Saliency-Aware Principle Selection} (SAP), a training-free inference-time method that replaces both rigid recipes with a search over a discrete principle space.
Concretely, SAP combines \textbf{\textit{(i)}} principle-guided reasoning generation, parameterizing reasoning behaviors via high-level principles; \textbf{\textit{(ii)}} evolutionary principle refinement, performing population-based selection under ordinal feedback; and \textbf{\textit{(iii)}} saliency-aware evaluation, supplying a modality-aware fitness signal.
An overview is illustrated in Fig.~\ref{fig:framework}.

\subsection{Principle-Guided Reasoning Generation}

\subsubsection{Intractable Route Search}
Direct optimization of the task utility $f$ over the reasoning route space $\mathcal{R}$ as shown in Eq. (\ref{eq:obj}) to achieve test-time scaling for VLM is intractable due to both theoretical and practical reasons.
First, we note that a reasoning route is not a fixed-length variable but the emergent output of autoregressive generation under the language model $\varphi$.
As inference-time computation is scaled, reasoning routes become longer and increasingly sparse, making direct sampling and evaluation more time-consuming in practice. In addition, the space $\mathcal{R}$ is inherently discrete and redundant.

Small local modifications to token sequences often do not lead to meaningful changes in the final reasoning outcome, while semantically distinct reasoning behaviors may be separated by large combinatorial gaps.
Therefore, information in the route space is sparsely distributed: many distinct routes induce similar final outputs, while high-quality routes are rare and difficult to identify.
This discreteness and sparsity make it difficult to efficiently compare, rank, or refine reasoning routes, rendering direct optimization over $\mathcal{R}$ impractical.
To make this problem amenable to optimization, we parameterize the routes using a set of high-level \emph{reasoning principles}.

\subsubsection{Principle-guided reasoning}

We start by defining \emph{principle reasoning} as a form of inference in which the reasoning process is guided by explicit, high-level principles that specify \emph{how} the model should reason, rather than the exact intermediate tokens to be generated.
Compared to a specific route, a reasoning principle provides guidance on reasoning behavior—such as how visual evidence should be revisited during the generation process, how intermediate hypotheses should be verified, or how different modalities should be balanced—without constraining to a particular token-level realization.

For example, in a vision–language reasoning task, a principle may instruct the model to \emph{explicitly re-examine visual evidence whenever an intermediate textual conclusion is formed} or to \emph{verify spatial relations by checking the relative positions of all relevant objects in the image}.
We note that such a principle $x$ is itself a compact textual description, but it does not prescribe a single reasoning trace.
Instead, under a fixed language model, the same principle $x \in \mathcal{X}$ may induce multiple concrete reasoning trajectories, depending on how the model instantiates the principle at the token level.
We denote each such instantiation as a reasoning route $\rho(x)$ given principle $x$.
This parameterization yields the equivalent optimization objective
\begin{equation}
\label{equ:formulation}
x^\star = \arg\max_{x \in \mathcal{X}} \; f\big(\rho(x)\big),
\end{equation}
which shifts inference-time optimization from unstructured routes to a compact principle space.

\subsubsection{Principle Selection}
Principle-guided reasoning shifts inference-time control from token-level trajectories to high-level reasoning directives.
This shift raises a central question: \emph{How should reasoning principles be selected and refined during inference?}

In VLMs, principle selection is challenging for two reasons closely tied to multimodal inference.
First, during long-horizon generation, visual evidence becomes increasingly difficult to revisit, making it necessary to favor principles that explicitly preserve visual grounding throughout reasoning.
In addition, feedback available at inference time is inherently discrete and ordinal, providing only relative comparisons between reasoning behaviors (e.g., preferring A to B) rather than forcing them into scalar scores.

These properties indicate that principle selection should operate directly over discrete candidates, rely on relative comparisons instead of absolute scores, and remain robust under noisy and inconsistent feedback.
This principle-level view also separates behavioral search from token-level realization: selection decides which reasoning directive should survive, while the frozen model remains free to instantiate that directive through different routes.
Maintaining a population further preserves competing behaviors long enough to compare their visual support, rather than committing early to a single chain that may already have drifted.
Motivated by these requirements, we adopt a population-based evolutionary selection strategy to refine reasoning principles. The next subsection instantiates this strategy over the principle space, with visual saliency providing the only modality-aware supervision.

%% file: sec/2_mtd/mtd_camera_ready.tex
% \subsection{Overview and Preliminary}
% \input{sec/2_mtd/sec/pre_v3}
% \subsection{Formulation: Principle Parameterization}
% \input{sec/2_mtd/sec/formulation_v2}

% The framework / pipeline figure now lives on page 1 (wired in the
% entry-point .tex right after the abstract). Do not re-include here.
\begin{figure}[t]
\centering
\begin{minipage}[t]{0.40\linewidth}
\input{sec/2_mtd/alg/alg_init_v2}
\end{minipage}\hfill
\begin{minipage}[t]{0.58\linewidth}
\input{sec/2_mtd/alg/alg_evo_v2}
\end{minipage}
\end{figure}
\subsection{Saliency-Aware Principle Evolution for Vision--Language Reasoning}
\label{sec:sap_framework}
\begingroup
\setlength{\intextsep}{2pt}
\begin{wraptable}[6]{r}{0.5\textwidth}
\centering
\small
\setlength{\tabcolsep}{4pt}
\renewcommand{\arraystretch}{1.3}
% \caption{Notation.}
\label{tab:sap-notation}
\begin{tabular}{c l}
\hline
$p$, $x$ & prompt [input], reasoning principle [variable] \\
$r$, $s$ & reasoning [thinking] and summary [output] \\
$\mu$, $\lambda$, $\tau$ & \#elites, \#new ones, \#routes [hyperparam.] \\
$\varphi$, $f$ & LLM, utility and fitness metric on ($r$, $s$) \\
\hline
\end{tabular}
\end{wraptable}
\input{sec/2_mtd/sec/framework_SAP_camera_ready}
\endgroup
% \subsection{Framework: Principle Evolutionary Reasoning}
% \input{sec/2_mtd/sec/framework}
% \subsection{Loss Design: Saliency-Aware Principle Selection}
% \input{sec/2_mtd/sec/SAP_v2}
\subsection{Framework Analysis}
\input{sec/2_mtd/sec/analysis_camera_ready}

%% file: sec/2_mtd/alg/alg_init_v2.tex
\begin{algorithm}[H]
\caption{Initialization: sample principles and $\tau$ routes}
\label{alg:init}
\begin{algorithmic}[1]
\REQUIRE images $\mathcal{I}$, prompt $p$, \#init. $N$, \#routes $\tau$
\STATE $\mathcal{O} \leftarrow \textsc{Ground}(\mathcal{I})$, $x \leftarrow [\,]$
\FOR{$i=1$ \textbf{to} $N$}
    \STATE $x_i \leftarrow \varphi_x(p, \mathcal{O})$ \COMMENT{sample one principle}
    \STATE $r_i, s_i \leftarrow [\,], [\,]$
    \FOR{$j=1$ \textbf{to} $\tau$}
        \STATE $r_{i,j}, s_{i,j} \leftarrow \varphi_{rs}(p, \mathcal{I}, x_i)$
        \STATE $\textsc{append}(r_i, r_{i,j})$ \;
        \STATE $\textsc{append}(s_i, s_{i,j})$ \;
    \ENDFOR
    \STATE $\textsc{append}(x, x_i)$
\ENDFOR
\STATE \textbf{return} $x,\;\{r_{i,j}, s_{i,j}\}_{i\in[N],\,j\in[\tau]}$
\end{algorithmic}
\end{algorithm}

%% file: sec/2_mtd/alg/alg_evo_v2.tex
\begin{algorithm}[H]
\caption{$(\mu+\lambda)$ Evolution: select elites and resample}
\label{alg:evolve}
\begin{algorithmic}[1]
\REQUIRE principles $x$, images $\mathcal{I}$, prompt $p$, \#routes $\tau$
\STATE $N \leftarrow |x|$ \COMMENT{$N=\mu+\lambda$}, $f \leftarrow [\,]$
\FOR{$i=1$ \textbf{to} $N$}
    \STATE $\{(r_{i,j}, s_{i,j})\}_{j\in[\tau]} \leftarrow \textsc{Inference}(p, \mathcal{I}, x_i, \tau)$
    \STATE $\ell_i \leftarrow \textsc{DiscreteEval}(\{(r_{i,j}, s_{i,j})\}_{j\in[\tau]})$
    \STATE $f_i \leftarrow \textsc{MapLevelToScore}(\ell_i)$ \COMMENT{internal only}
    \STATE $\textsc{append}(f, f_i)$
\ENDFOR
\STATE $\texttt{idx} \leftarrow \textsc{TopK}(f, k=\mu)$
\STATE $x^{\mu} \leftarrow x[\texttt{idx}]; x^{\lambda} \leftarrow [\,]$
\FOR{$n=1$ \textbf{to} $N-\mu$}
    \STATE $x_n \leftarrow \varphi_x(p, x^{\mu})$ \COMMENT{offspring from elites}
    \STATE $\textsc{append}(x^{\lambda}, x_n)$
\ENDFOR
\STATE $x^{+} \leftarrow \textsc{concat}(x^{\mu}, x^{\lambda})$
\STATE \textbf{return} $x^{+}$ new population of principles
\end{algorithmic}
\end{algorithm}

%% file: sec/2_mtd/sec/framework_SAP_camera_ready.tex
Built on the principle-level formulation in Eq.~(\ref{equ:formulation}), we cast inference-time scaling as a saliency-guided evolutionary search over reasoning principles. The framework combines \textbf{\textit{(i)}} principle-level reasoning that discovers routes, \textbf{\textit{(ii)}} visual saliency as the modality-aware supervision signal, and \textbf{\textit{(iii)}} population-based evolution as the optimization mechanism.

\refstepcounter{subsubsection}
\noindent\textbf{\thesubsubsection\quad Approach overview}\par\nobreak
Given an image $\mathcal{I}$ and prompt $p$, a frozen grounding tool extracts salient elements
\(
\mathcal{O} = \textsc{Ground}(\mathcal{I}),
\)
used as a modality-aware reference for evaluation rather than injected into the language model. Conditioned on $(p, \mathcal{O})$, we initialize a population of $\mu+\lambda$ principles $\{x_i\}_{i \in [\mu+\lambda]}$. Each principle $x_i$ instantiates $\tau$ routes via
$(r_{i,j}, s_{i,j}) \sim \varphi(r, s \mid x_i, p, \mathcal{I}), \quad j \in [\tau],$
realizing the reasoning routes $\rho(x_i)$ induced by the principle. Visual information thus remains available throughout inference, while different principles induce distinct reasoning behaviors.

\input{sec/2_mtd/fig/fig_evo_loop_camera_ready}
\Needspace{5\baselineskip}
\subsubsection{Saliency-aware fitness and evolution}
\label{sec:consensus}
The primary supervision is \textbf{evidence validity} $e_i$: the fraction of cited visual references in $\{r_{i,j}\}_{j\in[\tau]}$ that match a grounded object in $\mathcal{O}$, capturing how consistently the principle's routes refer to detected visual evidence. We use ``saliency'' operationally to mean this object-grounding consistency; it is neither a model of human visual attention nor a claim that the detected objects are necessarily task-relevant. This is the only signal derived from the grounding tool and drives the search toward visually supported principles. At \emph{initialization only}, the language model is given a textual report about $\mathcal{O}$ returned by the grounding tools---a compact list of object labels with image- and object-ids---used to seed the principle population; subsequent generations resample principles conditioned on elites and never re-receive $\mathcal{O}$. The $e_i$ score itself is computed \emph{externally} by canonicalizing each cited reference into the format \texttt{img\#k\_obj\#j($\ell_{k,j}$)} and matching it against the underlying label map. Thus, $\mathcal{O}$ enters the model only as an initialization-time textual prior, never as an authoritative answer or task supervision (Appendix~\ref{appdx:implementation}, Step~5).

Three further \emph{control} signals keep the search stable, each obtained from a lightweight external probe rather than from the reasoning trace itself:
\textbf{(a) consensus} $c_i$---a lightweight external judge compares each principle's representative answer against the population majority $\bar y$ and emits an ordinal agreement label (anchor against drift);
\textbf{(b) within-principle diversity} $d_i$---the same judge counts the number of distinct topics covered by the $\tau$ routes under $x_i$ and discretizes that count into an ordinal level (anchor against collapse to one route);
\textbf{(c) uncertainty} $u_i$---after each route the reasoning model is required to self-emit an ordinal certainty rating reflecting how confident it is in its own answer (suppressing over-confident or ambiguous outputs).
All four signals are ordinal in $\{\texttt{low},\texttt{medium},\texttt{high}\}$, reflecting the inherently ordinal and contrastive nature, and combine into a scalar fitness (full procedure in Appendix~\ref{appdx:implementation}, Steps~3--6):
\begin{equation}
\label{eq:sap_fitness}
f_i \;=\; \underbrace{w_e\,\mathrm{Score}(e_i)}_{\text{saliency-driven}} \;+\; \underbrace{w_c\,\mathrm{Score}(c_i) + w_d\,\mathrm{Score}(d_i) - w_u\,\mathrm{Pen}(u_i)}_{\text{control: anti-drift, anti-collapse, anti-overconfidence}},
\end{equation}
where $\mathrm{Score}$ and $\mathrm{Pen}$ are ordinal mappings $(\texttt{low},\texttt{medium},\texttt{high})\!\to\!(0,1,2)$.

A $(\mu+\lambda)$ selection scheme retains the top-$\mu$ principles as elites and resamples $\lambda$ new principles conditioned on the elites; iterating $T$ times yields a population-based approximation to the principle-level objective in Eq.~(\ref{equ:formulation}), allocating inference-time computation to progressively favor visually supported principles. Details: Alg.~\ref{alg:init},~\ref{alg:evolve}. We next analyze the loop convergence and computation.

%% file: sec/2_mtd/fig/fig_evo_loop_camera_ready.tex
% Editable source: figure3_sap_framework_editable.pptx (same directory).
\begin{wrapfigure}{r}{0.50\textwidth}
\centering
\includegraphics[width=0.48\textwidth]{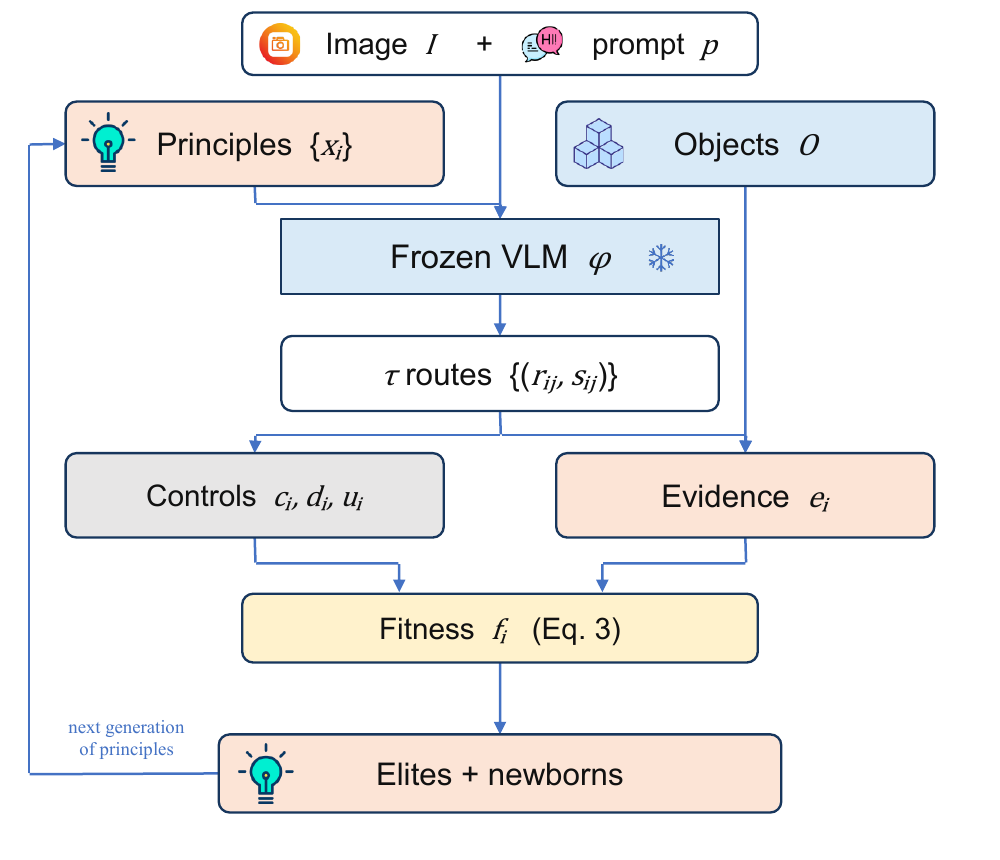}
\caption{\textbf{Saliency-aware evolutionary loop.}
Each principle generates parallel routes. Grounded object support and route-level control signals define fitness; elite principles and newborns form the next population for $T$ rounds.}
\label{fig:evo_loop}
\end{wrapfigure}

%% file: sec/2_mtd/sec/analysis_camera_ready.tex
In this part, we provide theoretical analysis that quantifies how inference-time scaling operates in SAP by explicitly characterizing, in terms of $(\mu,\lambda,T)$, the probability of optimization improvement and the coverage of effective reasoning principles achieved by the evolutionary search, where $T$ is the total number of iterations. Under the evolutionary algorithm with elite selection, iterative filtering ensures that the population fitness is non-decreasing across generations, such that the best-performing principles monotonically improve as optimization proceeds.

\input{sec/2_mtd/thm/theorem_opt_gen}

Theorem~\ref{thm:opt_gen} explains how inference-time scaling emerges in SAP.
The optimization guarantee ensures stability: elite preservation prevents degeneration by retaining the best-performing principles across iterations.
At the same time, the improvement bound shows that increasing $\lambda$ raises the probability of discovering strictly better principles at each iteration.
Beyond optimization, the generalization bound characterizes the growing \emph{coverage of the principle space}.
Since each iteration evaluates $\lambda$ newly sampled principles, after $T$ iterations the search explores on the order of $\lambda T$ distinct candidates.
As $\lambda$ or $T$ increases, the probability of encountering effective principles outside the current population grows monotonically.
Together, these results show that inference-time scaling in SAP is achieved by allocating additional computation to explore and select among alternative reasoning principles, rather than extending a single reasoning trajectory.
Formal proofs are provided in Appendix~\ref{appdx:framework_analysis}.

\subsubsection{Comparison with long chain-of-thought}
SAP further exhibits a scaling advantage over traditional long chain-of-thought (CoT) reasoning by operating in a parallel and population-based manner. 
Consider a fixed inference-time token budget $L$. 
In long CoT reasoning, tokens are generated sequentially along a single trajectory, and transformer-based decoding incurs a quadratic attention cost w.r.t. the sequence length, i.e., $\mathcal{O}(L^2)$, together with strict sequential dependencies that limit parallelism across decoding steps.
In contrast, SAP allocates the same \emph{route-token} budget across multiple shorter reasoning routes evaluated in parallel; complete system accounting, including grounding and selection overhead, is reported in Appendix~\ref{appdx:additional:response_time}.
Let $L = (\mu+\lambda)\tau \bar{\ell}$, where $\bar{\ell}$ denotes the average length of each reasoning route.
Each route incurs an attention cost of $\mathcal{O}(\bar{\ell}^2)$, yielding a total attention cost of
$\mathcal{O}\big((\mu+\lambda)\tau\, \bar{\ell}^2\big)$,
which is strictly lower than $\mathcal{O}(L^2)$ when $\bar{\ell} \ll L$.
Moreover, reasoning routes across different principles and routes are independent and can be executed in parallel, avoiding the sequential bottleneck inherent to long CoT.
Therefore, SAP scales inference-time reasoning by parallel exploration and selection rather than by elongating a single reasoning chain, providing a more computationally efficient alternative.

\subsubsection{Training-free scaling}
SAP is a plug-to-play, fully training-free method: it requires no gradient updates, no manually labeled data, and no fine-tuning of either the language model or the external grounding service. The grounding tool is consulted only at inference time and its outputs are advisory, so no external knowledge is absorbed into the language model's parameters (cf.\ Appendix~\ref{appdx:additional:trainfree}).
The evolutionary search is performed directly in the space of reasoning principles, using the fixed knowledge already encoded in the pretrained model, without introducing additional training data or external supervision.
Rather than acquiring new knowledge, SAP explores organizing and applying existing knowledge during inference.
This aligns closely with the nature of reasoning, which concerns how information is structured, combined, and verified, rather than what new information is learned~\citep{zhu2025reasoning}.
Consequently, SAP improves reasoning quality by reallocating computation to search for better reasoning behaviors within the model’s capabilities, avoiding the cost and potential bias of data collection or retraining.

%% file: sec/2_mtd/thm/theorem_opt_gen.tex
\begin{theorem}[\textbf{Optimization and generalization with $(\mu+\lambda)$ evolution}]
\label{thm:opt_gen}
Consider the $(\mu+\lambda)$ evolutionary procedure over a principle space $\mathcal{X}$.
Let $F_t$ denote the best fitness in the population at iteration $t$, and let $T$ be the total number of iterations.
Then the following properties hold:
\begin{enumerate}
    \item (\textbf{Optimization stability}) Due to elite selection, the best fitness is non-decreasing: $F_{t+1} \ge F_t, \ \forall t$.
    Moreover, if each offspring sample has probability at least $q_t$ of instantiating a principle with fitness strictly greater than $F_t$, then the probability of improvement in one iteration satisfies $\Pr(F_{t+1} > F_t) \ge 1 - (1 - q_t)^{\lambda}$.
    \item (\textbf{Generalization as coverage}) Suppose there exists a set of effective principles $\mathcal{X}_{\mathrm{good}} \subset \mathcal{X}$, and each offspring sample has probability at least $q>0$ of belonging to $\mathcal{X}_{\mathrm{good}}$.
    Then the probability that the evolutionary procedure discovers at least one effective principle within $T$ iterations is lower bounded by
   $1 - (1 - q)^{\lambda T}$.
\end{enumerate}
\end{theorem}

%% file: sec/3_exp/exp_camera_ready.tex
% \subsection{Experiment Setups}
\input{sec/3_exp/sec/setting_camera_ready}

\subsection{Comparison}
\input{sec/3_exp/sec/comparison_camera_ready}
\begingroup
\setlength{\intextsep}{2pt}
\setlength{\floatsep}{2pt}
\input{sec/3_exp/tbl/tbl_hallucination_camera_ready}
\input{sec/3_exp/tbl/tbl_comparison_camera_ready}
\endgroup

\subsection{Ablation Study}
\input{sec/3_exp/sec/ablation}

%% file: sec/3_exp/sec/setting_camera_ready.tex
We evaluate SAP on 16 vision--language benchmarks via VLMEvalKit~\citep{duan2024vlmevalkit}, grouped into \emph{Reasoning / Knowledge / General} and \emph{OCR / Document / Diagram / Real-world} (Appendix~\ref{appdx:bench}).
We separately evaluate object hallucination on POPE~\citep{li2023evaluating}, HallusionBench~\citep{guan2024hallusionbench}, and AMBER-d~\citep{wang2023amber}.
The benchmarks target single-pass VL grounding and reasoning, matching the scope declared in Section~\ref{sec:intro}; multi-turn conversational agent benchmarks include many non-perception turns and would dilute the perception signal we measure.
All experiments are pure inference: no parameter updates, no per-dataset tuning.

\textbf{Models.}
Qwen3-VL-8B~\citep{bai2025qwen3vl} is the default, frozen backbone for all reasoning, principle generation, and aggregation steps; cross-architecture transfer additionally uses InternVL3.5-8B~\citep{wang2025internvl3} and DeepSeek-VL2 (4.5B MoE)~\citep{wu2024deepseek}, and the model-scale sweep covers Qwen3-VL-Instruct at 2B/4B/8B/30B.
Visual grounding uses a frozen Segment-Anything Model (SAM)~\citep{kirillov2023segment} with labels from a lightweight captioner; the same backbone serves as the consensus/diversity judge (details in Appendix~\ref{appdx:implementation}).

\textbf{Inference procedure.}
Given prompt $p$ and image(s) $\mathcal{I}$, SAP keeps a population of $\mu+\lambda$ principles.
Each principle $x_i$ generates $\tau$ reasoning routes in one forward call; fitness follows Eq.~(\ref{eq:sap_fitness}); the population evolves by $(\mu+\lambda)$ selection.
After the last generation an optional single aggregation call synthesizes the final answer.

\textbf{Hyperparameters.}
$\mu=\lambda=2$, $\tau=2$, $T=2$ generations, fitness weights $w_c=w_d=w_e=w_u=1$, fixed across all 16 benchmarks.
The number of grounded objects per image is capped; all discretization thresholds are fixed.

\textbf{Evaluation.}
Each benchmark uses its official protocol (exact match for free-form, label mapping for multiple choice).
\emph{Average performance} (avg.\ perf.) is the mean across the 16 benchmarks excluding MME, whose scale differs.

\paragraph{Baselines.}
\textbf{(i)} \emph{Instruct}: direct answer from $(\mathcal{I},p)$;
\textbf{(ii)} \emph{Thinking}: long chain-of-thought (LongCoT) with explicit reasoning traces;
\textbf{(iii)} \emph{SAP} (ours).
All share the same checkpoint and decoding settings.
\MDF{To isolate principle search from the grounding signal itself, we additionally compare against \emph{SAM-Augmented LongCoT} (with and without our advisory prompt) and representative inference-time search / visually-grounded baselines: Self-Consistency~\citep{wang2022self}, Tree of Thoughts~\citep{yao2023tree}, Multimodal-CoT~\citep{zhang2023multimodal}, VGR~\citep{wang2025vgr}, and TreeVGR~\citep{wang2026treevgr}; see Appendix~\ref{appdx:additional:samaug} and~\ref{appdx:additional:sota}.}

%% file: sec/3_exp/sec/comparison_camera_ready.tex
Tab.~\ref{tab:hallucination_main} compares SAP with training-free decoding baselines under a unified Qwen3-VL-8B evaluation stack.
SAP obtains the strongest POPE precision and absent-object recall, and improves HallusionBench and AMBER-d while keeping POPE accuracy and F1 close to the Instruct backbone.
This pattern supports the targeted claim that SAP reduces unsupported object assertions rather than uniformly improving every VLM capability.
Table~\ref{tab:main_benchmark} reports the full 16-benchmark comparison across perception, reasoning, OCR, document, diagram, and real-world tasks. The averages in the subsequent ablations summarize their corresponding expanded appendix experiments; they are not averages over the columns of either main table.

%% file: sec/3_exp/tbl/tbl_hallucination_camera_ready.tex
\begin{table}[H]
\centering
\tablestyle{2pt}{1.3}
\footnotesize
\caption{Object-hallucination results under a unified Qwen3-VL-8B stack.
All methods use the same benchmark items, evaluator interfaces, and decoding budget.
All comparison methods are training-free decoding baselines; error bars are shown for Instruct and SAP over three seeds.
$R'$ denotes recall on absent-object queries.}
\label{tab:hallucination_main}

\begin{tabular*}{\textwidth}{@{\extracolsep{\fill}}lrrrrrr@{}}
method
& \multicolumn{4}{c}{POPE}
& HallusionBench
& AMBER-d \\
\cline{2-5}
& precision & $R'$ & accuracy & F1 & score & score \\
\hline
Instruct
& 89.6$\pm$2.2 & 90.0$\pm$2.4 & \textbf{88.0$\pm$1.9} & \textbf{87.8} & 72.7$\pm$2.5 & 88.7$\pm$2.5 \\
MFCD~\citep{liu2025mfcd}
& 87.4 & 87.3 & 87.7 & 87.7 & 72.0 & 87.3 \\
M3ID~\citep{favero2024m3id}
& 81.5 & 80.0 & 84.0 & 84.6 & 65.3 & 83.3 \\
ICD~\citep{wang2024icd}
& 79.0 & 76.7 & 82.3 & 83.3 & 64.3 & 64.0 \\
MaskCD~\citep{deng2025maskcd}
& 86.3 & 86.7 & 85.3 & 85.1 & 73.3 & 78.7 \\
\hline
\textbf{SAP (ours)}
& \textbf{91.2$\pm$1.3}
& \textbf{92.0$\pm$1.2}
& 87.6$\pm$1.1
& 87.0$\pm$1.1
& \textbf{80.6$\pm$1.3}
& \textbf{92.4$\pm$1.9} \\
\end{tabular*}
\end{table}

%% file: sec/3_exp/tbl/tbl_comparison_camera_ready.tex
\begin{table}[H]
\centering
\tablestyle{10pt}{1.3}
\caption{Comparison of longer and wider reasoning across 16 vision--language benchmarks under the same Qwen3-VL-8B backbone and inference-time decoding budget.
Set I emphasizes perception and reasoning; Set II covers specialized visual inputs.
Thinking denotes sequential LongCoT, whereas SAP scales over parallel routes.
Shaded cells identify the entries included in each method's average; colors distinguish Instruct, Thinking, and SAP, while the differently scaled MME tasks are excluded.
$\Delta$ is SAP minus Instruct.}
\label{tab:main_benchmark}

\newcommand{\avgI}[1]{\cellcolor{baselinebg}#1}
\newcommand{\avgT}[1]{\cellcolor{refbg}#1}
\newcommand{\avgS}[1]{\cellcolor{default}#1}
\begin{tabular}{p{0.29\textwidth}>{\centering\arraybackslash}p{0.115\textwidth}>{\centering\arraybackslash}p{0.115\textwidth}>{\centering\arraybackslash}p{0.115\textwidth}>{\raggedleft\arraybackslash}p{0.075\textwidth}}
benchmarks ($\uparrow$) & Instruct & Thinking & \textbf{SAP (ours)} & $\Delta$ \\
\hline
\textbf{avg. except MME} & \avgI{75.4} & \avgT{75.9} & \avgS{\textbf{76.6}} & \textcolor{blue!80}{\textbf{+1.2}} \\
\hline
\multicolumn{5}{c}{\emph{Set I. Perception / Reasoning / General Knowledge}} \\
\hline
POPE-recall~\citep{li2023evaluating} & \avgI{83.9} & \avgT{79.6} & \avgS{\textbf{89.9}} & \textcolor{blue!80}{+6.0} \\
MME-p~\citep{fu2025mme} & 1744.7 & 1663.2 & \textbf{1755.5} & \textcolor{blue!80}{+10.8} \\
MME-r~\citep{fu2025mme} & 656.7 & \textbf{722.5} & 689.9 & \textcolor{blue!80}{+33.2} \\
MMBench~\citep{liu2024mmbench} & \avgI{70.7} & \avgT{\textbf{87.0}} & \avgS{82.9} & \textcolor{blue!80}{+12.2} \\
SEEDBench~\citep{li2023seed} & \avgI{68.5} & \avgT{\textbf{78.3}} & \avgS{77.3} & \textcolor{blue!80}{+8.8} \\
MMMU~\citep{yue2024mmmu} & \avgI{\textbf{62.4}} & \avgT{62.3} & \avgS{62.3} & \textcolor{red!60}{--0.1} \\
\hline
\multicolumn{5}{c}{\emph{Set II. Science / Text / OCR / Document / Diagram / Real-world}} \\
\hline
ScienceQA~\citep{lu2022scienceqa} & \avgI{88.0} & \avgT{\textbf{92.5}} & \avgS{92.2} & \textcolor{blue!80}{+4.2} \\
TextVQA~\citep{singh2019towards} & \avgI{\textbf{82.3}} & \avgT{77.2} & \avgS{81.3} & \textcolor{red!60}{--1.0} \\
OCRVQA~\citep{mishra2019ocrvqa} & \avgI{\textbf{63.2}} & \avgT{44.1} & \avgS{62.8} & \textcolor{red!60}{--0.4} \\
OCRBench~\citep{liu2024ocrbench} & \avgI{56.8} & \avgT{\textbf{65.9}} & \avgS{60.2} & \textcolor{blue!80}{+3.4} \\
ChartQA~\citep{masry2022chartqa} & \avgI{\textbf{82.7}} & \avgT{77.2} & \avgS{78.1} & \textcolor{red!60}{--4.6} \\
DocVQA~\citep{mathew2021docvqa} & \avgI{\textbf{92.2}} & \avgT{91.6} & \avgS{90.6} & \textcolor{red!60}{--1.6} \\
InfoVQA~\citep{mathew2022infographicvqa} & \avgI{82.6} & \avgT{\textbf{84.1}} & \avgS{83.1} & \textcolor{blue!80}{+0.5} \\
AI2D~\citep{kembhavi2016diagram} & \avgI{79.2} & \avgT{\textbf{82.2}} & \avgS{80.1} & \textcolor{blue!80}{+0.9} \\
BLINK~\citep{fu2024blink} & \avgI{\textbf{68.2}} & \avgT{60.3} & \avgS{62.6} & \textcolor{red!60}{--5.6} \\
RealWorldQA~\citep{realworldqa} & \avgI{70.3} & \avgT{72.8} & \avgS{\textbf{73.3}} & \textcolor{blue!80}{+3.0} \\
\end{tabular}
\end{table}

%% file: sec/3_exp/sec/ablation.tex
All ablations are reported relative to the default SAP configuration and the corresponding \emph{Instruct} and LongCoT \emph{Thinking} baselines.

\input{sec/3_exp/tbl/tbl_ablation_combined}

\textbf{Principle gain vs.\ depth (Tab.~\ref{tab:additional:adherence_marginal}).} The marginal of the $(\tau,L)$ grid in Appendix~\ref{appdx:additional:grid} shows that the no-principle POPE collapses from $84.6$ at $L/8$ to $78.9$ at $L{=}4096$, while the with-principle score holds at $87.9$---so the absolute POPE gap is largest at full length, the opposite of what instruction-following degradation would produce. Principles compensate most where visual decay is most severe. Means over 3 seeds; error bars in Appendix~\ref{appdx:additional:adherence}.
\textbf{Modules of SAP (Tab.~\ref{tab:ablation_modules}).} Adding routes without evolution does not improve upon Instruct---route diversity alone is not enough. Adding evolution without multi-route collapses into single-trajectory optimization and degrades performance. Only the full configuration recovers the gain; replacing elite selection with naive aggregation undoes it, supporting the claim that explicit principle filtering---not route averaging---is the operative mechanism.
\textbf{Fitness weights (Fig.~\ref{fig:ablation_weight}).} Zeroing out any single fitness component drops performance, so the four ordinal signals encode necessary, non-redundant constraints. A smooth $w_d/w_c$ interpolation (Appendix~\ref{appdx:additional:weightsmooth}) is flat within $\sim$1.5 points across the interior, so SAP keeps $w_c{=}w_d{=}w_e{=}w_u{=}1$ across all 16 benchmarks without per-dataset tuning.
\textbf{Cross-architecture (Tab.~\ref{tab:additional:architecture}).} On the two non-Qwen backbones listed under \emph{Models} above, SAP improves both on MMMU and MMBench, indicating that the gain is backbone-agnostic.
\textbf{Evolutionary hyperparameters (Tab.~\ref{tab:evo_hyper}).} Within the explored range, increasing the elite size $\mu$ does not consistently help while increasing the offspring count $\lambda$ produces small gains. This matches Theorem~\ref{thm:opt_gen}, which predicts improvement to scale with $\lambda$ rather than $\mu$.
\textbf{Scales (Tab.~\ref{tab:model_scales}).} Sweeping Qwen3-VL-Instruct from 2B to 30B, SAP improves consistently; the largest absolute gain is observed on the smallest backbone---the regime where unconstrained generation is most prone to drift.

\input{sec/3_exp/fig/fig_ablation_weights}

\MDF{\textbf{Robustness to grounding quality.}
Three controlled perturbations target the input, representation, and output of
the grounding pipeline (Appendix~\ref{appdx:additional:robust}). SAP remains
above the ungrounded Instruct baseline under substantial image noise, up to
70\% feature-token drop, and even complete removal of the grounding summary.
In contrast, directly appending clean grounding to LongCoT fails sharply.
Thus, the gain does not require unusually favorable detections: the advisory
interface and route-level cross-validation matter more than exposing the model
to additional object tokens.

\textbf{Search controls and statistical stability.}
Under matched backbone and token budgets, SAP outperforms Self-Consistency,
Tree of Thoughts, and Multimodal-CoT on the perception-sensitive comparison,
while remaining competitive with trained visual-grounding methods
(Appendix~\ref{appdx:additional:sota}). Three-seed deviations are smaller than
the main inter-method gaps (Appendix~\ref{appdx:additional:variance}). On
MathVista and MathVerse, longer SAP routes match or exceed LongCoT, showing
that grounding-aware selection does not preclude reasoning depth
(Appendix~\ref{appdx:additional:mathbench}). Finally, independent routes reduce
wall-clock latency once parallel instances are available, although total token
compute is unchanged (Appendix~\ref{appdx:additional:response_time}). Together
with the complete $(\tau,L)$ grid and smooth weight sweep, these controls rule
out chain shortening, random variation, and a narrowly tuned operating point
as explanations for the gain.}

%% file: sec/3_exp/tbl/tbl_ablation_combined.tex
% Composite ablation table: 5 sub-panels.
%   Top row (~0.48\textwidth):    (a) principle gain vs. chain length ; (b) module ablation
%   Bottom row (~0.32\textwidth): (c) cross-architecture (InternVL3.5 / DeepSeek-VL2) ;
%                                 (d) $(\mu,\lambda)$ sweep ; (e) Qwen3-VL scales
% Sub-labels preserved: tab:additional:adherence_marginal, tab:ablation_modules,
%                       tab:additional:architecture, tab:evo_hyper, tab:model_scales.
% The smooth $w_d/w_c$ weight sweep is moved to Appendix~\ref{appdx:additional:weightsmooth}.

\begin{table}[t]
\centering
\caption{Five SAP ablations study principle depth, modules, architecture transfer, evolutionary hyperparameters, and model scale.
\colorbox{reprobg}{Light gray} marks panel references/readouts; \colorbox{default}{Dark gray} marks default full SAP (76.6).
Averages are appendix-specific, not main-table averages.}
\label{tab:ablation}
% ---------- (a) Principle gain vs. chain length ----------
\begin{minipage}[t]{0.44\textwidth}
\centering
\tablestyle{1.5pt}{1.3}
\subcaption{\textbf{Principle gain vs.\ chain length.} POPE gap widens with depth w/o vs.\ w/ principles.}
\label{tab:additional:adherence_marginal}
\begin{tabular}{lc|cc|cc|cc}
       &        & \multicolumn{2}{c|}{POPE} & \multicolumn{2}{c|}{OCRVQA} & \multicolumn{2}{c}{MME-R} \\
Length & Tokens & w/o & w/ & w/o & w/ & w/o & w/ \\
\hline
$L/8$ & 512  & 84.6 & 87.7 & 61.3 & 62.5 & 399.7 & 568.3 \\
$L/4$ & 1024 & 84.6 & 88.4 & 62.0 & 62.6 & 533.6 & 628.5 \\
$L/2$ & 2048 & 83.3 & 88.8 & 63.4 & 63.4 & 625.1 & 683.2 \\
$L$   & 4096 & 78.9 & \textbf{87.9} & 62.9 & 63.7 & 643.3 & 674.5 \\
\end{tabular}
\end{minipage}\hfill
% ---------- (b) Module ablation ----------
\begin{minipage}[t]{0.48\textwidth}
\centering
\tablestyle{3pt}{1.3}
\subcaption{\textbf{Module ablation.} Both evolution and multi-route are necessary.}
\label{tab:ablation_modules}
\begin{tabular}{lrrr>{\columncolor{reprobg}}r}
{settings} & {\#gen.\ evo.} & \#routes $\boldsymbol{\tau}$ & {decision} & {avg.} \\
\hline
instruct          & 0   & 1      & --    & 75.4 \\
w/o evolution     & 0   & $\tau$ & elite & 73.4 \\
w/o multi-route   & $T$ & 1      & elite & 73.0 \\
full SAP          & $T$ & $\tau$ & elite & \cellcolor{default}\textbf{76.6} \\
w/ summarizer     & $T$ & $\tau$ & agg.  & 74.8 \\
\end{tabular}
\end{minipage}

% ---------- (c) Cross-architecture transfer ----------
\begin{minipage}[t]{0.33\textwidth}
\centering
\tablestyle{2pt}{1.3}
\subcaption{\textbf{Cross-architecture (reasoning).}}
\label{tab:additional:architecture}
\begin{tabular}{l>{\columncolor{reprobg}}rr>{\columncolor{reprobg}}rr}
        & \multicolumn{2}{c}{InternVL3.5} & \multicolumn{2}{c}{DeepSeek-VL2} \\
        & base & +SAP & base & +SAP \\
\hline
MMMU    & 74.6 & \textbf{75.3} & 51.2 & \textbf{56.7} \\
MMBench & 83.4 & \textbf{84.6} & 79.5 & \textbf{79.8} \\
\end{tabular}
\end{minipage}\hfill
% ---------- (d) (mu, lambda) sweep ----------
\begin{minipage}[t]{0.32\textwidth}
\centering
\tablestyle{2pt}{1.3}
\subcaption{\textbf{$(\mu,\lambda)$ sweep}. Balanced exploration $>$ maximizing survival.}
\label{tab:evo_hyper}
\begin{tabular}{l>{\columncolor{reprobg}}rrrrr}
$\lambda$ & 2 & 2 & 4 & 4 & 6 \\
$\mu$     & 2 & 4 & 2 & 4 & 6 \\
\hline
avg.\ & \cellcolor{default}76.6 & 74.8 & \textbf{76.8} & 76.5 & \textbf{76.8} \\
\end{tabular}
\end{minipage}\hfill
% ---------- (e) Qwen3-VL scales ----------
\begin{minipage}[t]{0.32\textwidth}
\centering
\tablestyle{4pt}{1.3}
\subcaption{Qwen3-VL-Instruct, 2/4/8/30B. SAP improves all \textbf{scales}.}
\label{tab:model_scales}
\begin{tabular}{lrr>{\columncolor{reprobg}}rr}
sizes    & 2B   & 4B   & 8B                       & 30B \\
\hline
instruct & 60.4 & 75.4 & 75.4                     & 77.3 \\
~+SAP    & 66.6 & 75.6 & \cellcolor{default}76.6  & 78.1 \\
\end{tabular}
\end{minipage}
\end{table}

%% file: sec/3_exp/fig/fig_ablation_weights.tex
\begin{wrapfigure}[13]{r}{0.46\textwidth}
    \centering
    \includegraphics[width=\linewidth]{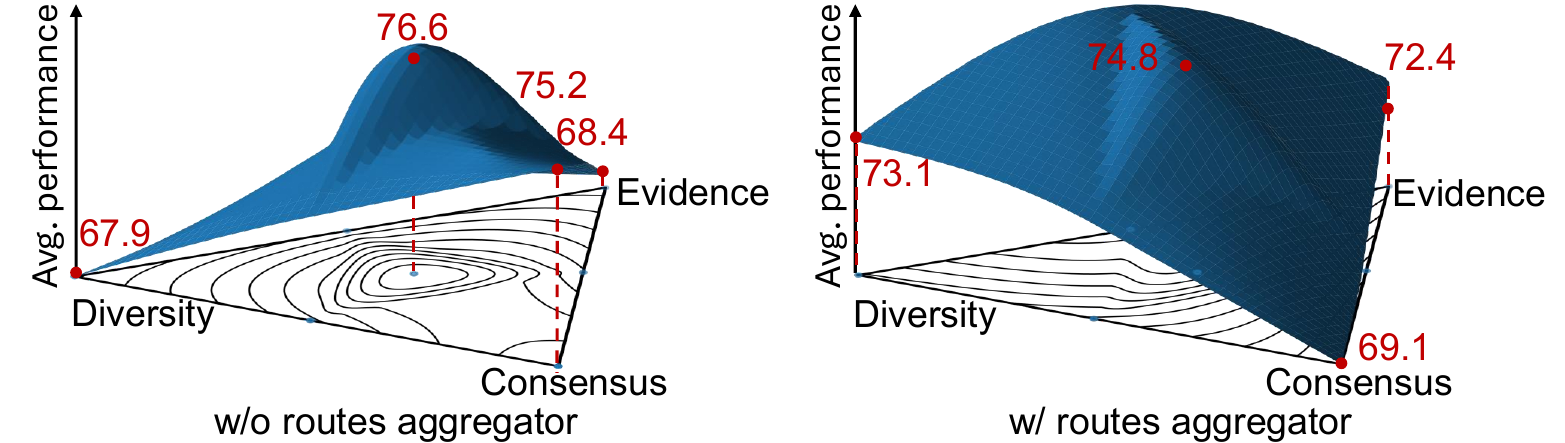}
    \caption{Fitness weights $w_d/w_c/w_e$ ($w_u{=}1$). Red values are appendix ablation averages, not main-table averages. Each term is necessary; the interior is stable.}
    \label{fig:ablation_weight}
\end{wrapfigure}

%% file: sec/4_rlt/rlt_camera_ready.tex
% We focus on \emph{vision--language (VL) grounding and reasoning}, where modality bias and hallucination are commonly mitigated by injecting cross-modal supervision~\citep{wang2024can,zhao2025r1,wang2025multimodal}.
% Two lines of work are most relevant to SAP; for each we name the single most important difference.

\paragraph{Inference-time search for VL reasoning.}
For multimodal trajectories, \citet{lin2025investigating} sample and tree-search multimodal chain-of-thought with a consistency verifier and \citet{wang2025scaling} train a vision value model that biases best-first decoding toward visually grounded continuations; for general prompts and programs, \citet{kojima2022large,pryzant2023automatic,wu2024evolutionary,van2024llamea} evolve a single prompt or LLM-generated program offline on a curated set and reuse it at every query. \textbf{Main difference:} all of these either deepen one trajectory or commit to one fixed artifact. SAP runs evolution \emph{per query at inference time} over multiple short principle-level routes launched in parallel from the original image and prompt. Each generation evaluates the resulting routes against the same frozen grounding reference and feeds that evidence into principle selection.

\paragraph{Tool-augmented grounded VL agents (the regime we target).}
\citet{wang2025vgr} fine-tunes a VLM on (region, reasoning) pairs so it emits bounding boxes the next step replays; \citet{wang2026treevgr} extends this to RL that jointly supervises localization and reasoning; \citet{zhao2024mitigating} feeds external object lists directly into decoding. \textbf{Main difference:} all three \emph{inject} tool outputs into the VLM's trajectory (by training the VLM to follow them, or by hard-conditioning decoding on them), so a single noisy call can dominate. SAP instead uses a frozen grounding reference to evaluate routes and select principles. The tool influences which reasoning behavior survives, but it does not write the routes or their answers.

\paragraph{Training-free hallucination mitigation.}
Recent decoding methods suppress visually unsupported tokens by contrasting the
original distribution with an image-perturbed or masked distribution
\citep{liu2025mfcd,deng2025maskcd}, grounding generation in visual information
\citep{favero2024m3id}, or contrasting instruction-conditioned predictions
\citep{wang2024icd}. These methods intervene at token selection within one
trajectory. SAP is complementary: it leaves token decoding unchanged and
selects among complete routes induced by different reasoning principles.
Consequently, its grounding signal can jointly assess object support,
cross-route agreement, diversity, and uncertainty, while rejected routes do
not contribute tokens to the surviving answer. It also permits black-box
grounding without retraining the VLM.

%% file: sec/5_sum/sum_camera_ready.tex
Long serial reasoning can weaken visual support, while directly inserting tool
outputs can turn a noisy detection into a persistent premise. SAP instead
scales inference over compact reasoning principles and uses grounding only to
select among independent routes. Across 16 benchmarks, it improves
perception-sensitive and object-hallucination results without sacrificing
broader reasoning capability; ablations, matched-budget baselines, controlled
grounding perturbations, and cross-model transfer identify principle selection
as the source of the gain. More broadly, tools can provide bounded evidence
without becoming authorities over the answer. SAP still requires multiple
model calls, benefits most from parallel hardware, and is evaluated with one
grounding tool in single-pass question answering. Multiple tools and
memory-aware selection for multi-turn agents are natural next steps.

%% file: sec/6_appdx/appdx_camera_ready.tex
\section{Additional Discussion}
\label{appdx:discussion}
\input{sec/6_appdx/sec/Discussion_camera_ready}

\section{Additional Empirical Analyses}
\label{appdx:additional}
\input{sec/6_appdx/sec/Additional_camera_ready}

\section{Implementation Details \& Code Report: Saliency-Aware Principle Selection}
\label{appdx:implementation}
\input{sec/6_appdx/sec/Implementation}

\subsection{Framework: Principle Evolutionary Reasoning}
\input{sec/2_mtd/sec/framework}

\section{Related Benchmarks}
\label{appdx:bench}
\input{sec/6_appdx/sec/Benchmarks}

\section{Algorithm Analysis}
\label{appdx:framework_analysis}
\input{sec/6_appdx/sec/Analysis_Framework}

\section{Template Design}
\input{sec/6_appdx/sec/Template}

% We put a case in the first figure, and cancel this part.
% \section{Case study}
% \input{sec/6_appdx/sec/Case}

% \newpage
% \input{sec/3_exp/tbl/tbl_comparison_vertical}

%% file: sec/6_appdx/sec/Discussion_camera_ready.tex
\paragraph{Limitations.} We note several limitations of the current approach for further research.
SAP is most effective and responds quickly only in large-scale or parallel inference settings, where its multi-route design can amortize computational costs.
In addition, its performance depends on the quality of the prompt templates of the evolutionary system and the principle of generation, which requires careful design.
Future work may explore meta-level principles for guiding principle generation itself, enabling more systematic control over diversity and randomness in the principle space.

\paragraph{Licenses and asset use.}
We use Qwen3-VL under Apache-2.0, InternVL3.5 under MIT, DeepSeek-VL2 code under MIT and its model weights under the DeepSeek Model License, Segment Anything under Apache-2.0, and VLMEvalKit under Apache-2.0. Benchmark data are accessed through their original distributions or VLMEvalKit and are not redistributed; each benchmark remains subject to the license and terms provided by its original authors.

\paragraph{Visual modality diminishes in textual autoregressive generation.} As illustrated in Fig.~\ref{fig:intro}, during long vision–language reasoning, the contribution of visual information tends to diminish as generation progresses.
While early reasoning steps may actively rely on visual evidence, later steps increasingly focus on the textual space.
This phenomenon reflects a fundamental discrepancy between modalities: textual representations are explicitly generated and updated at every decoding step, whereas visual inputs are typically incorporated only indirectly through fixed encodings or limited cross-modal interactions.
Once visual information is compressed into an early textual summary, any omissions or misinterpretations become difficult to correct in subsequent reasoning steps.
As inference-time computation is scaled by increasing generation length, such early grounding errors are therefore amplified rather than mitigated.
Beyond this modality imbalance, inference-time optimization in VLMs faces two intrinsic challenges.

\paragraph{Challenge I: inconsistent and ordinal signals.} First, supervision signals in multimodal reasoning are inherently inconsistent and ordinal.
Evaluations of vision–language reasoning outputs—whether provided by humans or models—are often based on heterogeneous and implicit criteria, particularly regarding how visual evidence should be used.
For the same reasoning outcome, different evaluators may emphasize different aspects, such as visual grounding, logical completeness, or conciseness.
Unlike language-only settings, where such inconsistency primarily affects stylistic preferences, in vision–language reasoning it directly influences whether and how visual evidence is incorporated.
As a result, supervision does not define a single well-calibrated scalar objective.
Instead, what is naturally available is an \emph{ordinal signal}: a partial ordering that indicates which reasoning behaviors are preferable to others, rather than forcing heterogeneous judgments into a single quantitative score.

\paragraph{Challenge II: sparse and discrete search space.} Second, the search space of inference-time reasoning is discrete and sparse.
Both token-level reasoning trajectories and higher-level reasoning principles are discrete objects drawn from a combinatorial space without a smooth local structure.
This challenge is further amplified in vision–language reasoning, where different reasoning behaviors may rely on different subsets of visual evidence, making local comparisons between candidates ill-defined.
Consequently, optimization strategies that rely on continuity assumptions or local gradient information are poorly suited for this setting.

Taken together, these properties suggest that effective inference-time optimization in VLMs should
(i) operate directly on discrete candidates,
(ii) rely on relative comparisons rather than absolute scalar supervision, and
(iii) remain robust under noisy and inconsistent feedback.
These requirements motivate a population-based, selection-driven optimization paradigm for principle-level reasoning, which we adopt in our approach.

%% file: sec/6_appdx/sec/Additional_camera_ready.tex
% =====================================================================
% Additional analyses extending Section~4. All subsections present
% native, self-contained empirical or methodological analyses.
% =====================================================================

\subsection{POPE-Recall and POPE-Overall: A Mathematical Note}
\label{appdx:additional:poperecall}

POPE consists of balanced positive and negative yes/no queries about object presence. The POPE-recall metric reported in Tab.~\ref{tab:main_benchmark} measures hit rate on the target class, which is the principal axis along which object-hallucination is evaluated: a higher recall on this class corresponds to better grounding fidelity for the agent's affirmative claims. With balanced classes the overall accuracy decomposes as
\[
A \;=\; \tfrac{1}{2}\bigl(R + R'\bigr),
\]
where $R$ is the reported POPE-recall and $R'$ is the recall on the complementary class. SAP's improvement on $R$ from $83.9$ to $89.9$ is a $\sim 37\%$ relative reduction in error on the targeted axis. The corresponding move on the overall accuracy $A$ is small because $R'$ absorbs part of the calibration shift; SAP optimizes for the grounding-fidelity axis captured by $R$, which is the property the abstract claim ``reduces object hallucination'' refers to. We report POPE-recall in the main comparison for this reason; the full per-class breakdown is omitted from the main table only for space and follows the same trend.

\subsection{Training-Free Property and the Role of External Tools}
\label{appdx:additional:trainfree}

SAP is training-free: it performs no gradient updates, requires no labeled data, and fine-tunes none of its components. The grounding service (e.g., SAM) and the captioner used to label its outputs are invoked strictly as frozen inference-time tools.

Because the language model never absorbs information from these tools into its parameters, the relationship between SAP and the external grounding service is not a form of model distillation: distillation would require the target model to internalize external knowledge into its weights via training. SAP instead queries the tool at inference time to obtain grounding signals on a per-input basis, which makes the tool consultation pattern closer to retrieval-augmented generation. The advisory treatment of tool outputs (Appendix~\ref{appdx:additional:robust}) further ensures that the language model overrides erroneous tool outputs rather than absorbing them.

\subsection{SAM-Augmented LongCoT Baseline}
\label{appdx:additional:samaug}

To isolate the contribution of evolutionary principle search from the contribution of the SAM-derived grounding summary itself, we evaluate a controlled baseline in which the identical grounding summary used by SAP is fed to the standard LongCoT pipeline. We additionally evaluate a version that further inherits the prompt skeleton used by SAP (advisory phrasing, evidence-hierarchy block) but without principle-level evolution or multi-route inference.

\begin{table}[h]
\centering
\small
\caption{SAM-augmented LongCoT, w/wo advisory prompt. Qwen3-VL-8B; POPE/OCRVQA/MME-R. Advisory prompt prevents catastrophic over-trust.}
\label{tab:additional:samaug}
\setlength{\tabcolsep}{4pt}
\begin{tabular}{lccc}
\toprule
Method & POPE (\%) & OCRVQA (\%) & MME-R \\
\midrule
Instruct                                & 83.9 & 63.2 & 656.7 \\
Thinking (LongCoT)                      & 79.6 & 44.1 & 722.5 \\
SAM-Aug.\ LongCoT                       & 14.8 & 45.2 & 328.4 \\
SAM-Aug.\ LongCoT w/ SAP prompt         & 54.8 & 62.2 & 623.1 \\
\textbf{SAP (ours)}                     & \textbf{89.9} & \textbf{62.8} & \textbf{689.9} \\
\bottomrule
\end{tabular}
\end{table}

Two findings follow. First, naively injecting SAM grounding into LongCoT without an advisory prompt is catastrophically unstable on POPE (14.8): the model over-trusts the textual summary and propagates a single grounding error throughout the entire reasoning chain---a single-point failure of tool use. Second, adding the advisory prompt partially recovers performance (POPE 54.8) but still leaves a substantial gap to SAP (89.9). The remaining gain is therefore directly attributable to evolutionary principle search and multi-route inference, not to the grounding signal alone.

\subsection{Smooth Sensitivity of Fitness Weights}
\label{appdx:additional:weightsmooth_summary}

The four-component fitness in Eq.~(\ref{eq:sap_fitness}) encodes two natural tensions---\emph{diversity vs.\ consensus} (routes should be diverse yet convergent) and \emph{evidence vs.\ uncertainty} (ground in visual facts, suppress uninformed guesses). Each term blocks a distinct failure mode. Figure~\ref{fig:ablation_weight} reports extreme configurations in which entire components are zeroed out; the resulting drops therefore measure necessity, not local sensitivity.

A complementary smooth interpolation between consensus and diversity (holding evidence fixed at $w_e=0.5$ and enforcing $\sum w = 1$) is reported as panel~(c) of Tab.~\ref{tab:ablation} in the main text. Performance degrades only at the two extremes (a single dominant component); intermediate configurations remain stable within $\sim$1.5 points. Throughout the paper SAP uses $w_c=w_d=w_e=w_u=1$ across all 16 benchmarks---no per-dataset weight search is performed.

\subsection{Principle Adherence under Increasing Reasoning Depth}
\label{appdx:additional:adherence}

A natural concern is whether the language model continues to adhere to the high-level principle as the chain-of-thought grows longer. We argue this failure mode is structurally absent in SAP and confirm the argument empirically.

\paragraph{Architectural argument.}
Each reasoning route in SAP is an independent forward call with the principle injected at position 0 of a fresh context, never competing with accumulating reasoning text. Each route never exceeds the length of standard LongCoT generation ($L=4096$ tokens); under the default configuration ($\tau=2$, $L/2=2048$ tokens per route), each route operates well within the typical LongCoT range while the \emph{total} budget is identical. The dilution failure mode---in which a principle injected once is gradually overshadowed---is a property of single-chain generation under fixed budget; SAP scales through width, not depth, and therefore does not exhibit it.

\paragraph{Empirical argument: principle gain grows with depth.}
If principle adherence degraded with chain length, the absolute POPE score with principles should \emph{converge} to the no-principle baseline as the chain length approaches the LongCoT budget. The marginal distribution from the $(\tau, L)$ interaction grid (Appendix~\ref{appdx:additional:grid}) is reported as panel~(a) of Tab.~\ref{tab:ablation} in the main text and shows the opposite: at $L=4096$ (matching standard LongCoT generation length), the no-principle POPE collapses to $78.9$ while the with-principle score holds at $87.9$, the largest absolute gap across all four depths. Principle gains grow precisely because visual decay is most severe at full length, and principles compensate most effectively there. The plateau in absolute POPE beyond $\tau=4$ also reflects the finite diversity of effective principles in practice rather than instruction-following failure. The full version with $\pm$ standard deviation across 3 seeds is shown below; main-text Tab.~\ref{tab:additional:adherence_marginal}~(a) reports means only for space.

\begin{table}[h]
\centering
\small
\setlength{\tabcolsep}{4pt}
\renewcommand{\arraystretch}{1.3}
\caption{Principle gain vs.\ chain length: absolute scores w/o vs.\ w/ principles, with $\pm$ standard deviation across 3 seeds (attached to the with-principle column; the no-principle baseline is deterministic across seeds).}
\label{tab:additional:adherence_full}
\begin{tabular}{lc|cc|cc|cc}
\toprule
       &        & \multicolumn{2}{c|}{POPE} & \multicolumn{2}{c|}{OCRVQA} & \multicolumn{2}{c}{MME-R} \\
Length & Tokens & w/o & w/ & w/o & w/ & w/o & w/ \\
\midrule
$L/8$ & 512  & 84.6 & $87.7{\pm}1.70$ & 61.3 & $62.5{\pm}0.87$ & 399.7 & $568.3{\pm}59.8$ \\
$L/4$ & 1024 & 84.6 & $88.4{\pm}2.85$ & 62.0 & $62.6{\pm}1.78$ & 533.6 & $628.5{\pm}55.4$ \\
$L/2$ & 2048 & 83.3 & $88.8{\pm}2.36$ & 63.4 & $63.4{\pm}1.60$ & 625.1 & $683.2{\pm}36.2$ \\
$L$   & 4096 & 78.9 & $\mathbf{87.9{\pm}1.61}$ & 62.9 & $63.7{\pm}0.74$ & 643.3 & $674.5{\pm}16.7$ \\
\bottomrule
\end{tabular}
\end{table}

\subsection{Performance on Reasoning-Heavy Benchmarks}
\label{appdx:additional:mathbench}

Consensus matching could in principle amplify systematic errors on hard problems where the model rarely produces the correct answer. We address this concern in two ways.

\paragraph{Joint operation of fitness signals.}
Consensus is one of four signals; Evidence Validity and Uncertainty Penalty operate jointly with it, so SAP is not equivalent to Self-Consistency. A principle that induces a perceptual error (e.g., misidentifying a circle as a square) is penalized by Evidence Validity and down-weighted independently of majority status, while the Uncertainty Penalty further suppresses low-confidence routes. If a VLM fundamentally cannot recognize a basic shape, no aggregation strategy can recover---but joint filtering on evidence consistency and confidence minimizes error \emph{propagation}.

\paragraph{Reasoning-heavy benchmarks.}
Beyond MMMU (already in the main table), we evaluate on MathVista~\citep{lu2023mathvista} and MathVerse~\citep{zhang2024mathverse}.

\begin{table}[h]
\centering
\small
\caption{Qwen3-VL-8B on MathVista/MathVerse. SAP scales to reasoning-heavy tasks at longer routes.}
\label{tab:additional:math}
\setlength{\tabcolsep}{4pt}
\begin{tabular}{lcc}
\toprule
Method & MathVista (\%) & MathVerse (\%) \\
\midrule
Instruct                          & 77.1 & 62.4 \\
Thinking (LongCoT)                & 81.1 & 77.8 \\
SAP, 8 routes, $1/8$ length each  & 78.4 & 70.5 \\
SAP, 8 routes, $1/4$ length each  & 81.7 & 75.8 \\
SAP, 8 routes, $1/2$ length each  & \textbf{82.0} & \textbf{78.8} \\
\bottomrule
\end{tabular}
\end{table}

SAP targets vision-sensitive reasoning (object hallucination), not deductive mathematical reasoning where visual inputs are typically sharp and the bottleneck is reasoning depth rather than perceptual grounding. The transferability shown above is therefore a positive but secondary result.

\subsection{Disentangling Route Width and Chain Length}
\label{appdx:additional:grid}

To rule out the possibility that SAP's gains arise from chain-shortening or route-count alone, we report the complete 15-configuration $(\tau, L)$ interaction grid below, with and without saliency-aware principles, under the same total compute envelope.

\begin{table}[h]
\centering
\small
\caption{$(\tau, L)$ grid, w/wo principles. Qwen3-VL-8B; POPE/OCRVQA/MME-R. Principles drive the gain at every cell.}
\label{tab:additional:grid}
\setlength{\tabcolsep}{3pt}
\begin{tabular}{cc|ccc|ccc|ccc}
\toprule
$\tau$ & $L$ & POPE & POPE & $\Delta$ & OCR & OCR & $\Delta$ & MME-R & MME-R & $\Delta$ \\
       &     & w/o  & w/   & POPE     & w/o & w/  & OCR      & w/o   & w/    & MME-R \\
\midrule
1 & $L/8$ & 85.1 & 85.8 & $+0.7$  & 61.2 & 61.1 & $-0.1$ & 398.2 & 482.3 & $+ 84.1$ \\
1 & $L/4$ & 86.2 & 86.0 & $-0.2$  & 62.2 & 60.4 & $-1.8$ & 554.1 & 576.1 & $+ 22.0$ \\
1 & $L/2$ & 83.9 & 86.1 & $+2.2$  & 63.2 & 61.4 & $-1.8$ & 656.7 & 668.2 & $+ 11.5$ \\
1 & $L$   & 76.3 & 85.6 & $+9.3$  & 62.6 & 62.5 & $-0.1$ & 643.1 & 652.4 & $+  9.3$ \\
\midrule
2 & $L/8$ & 83.9 & 88.2 & $+4.3$  & 60.6 & 62.4 & $+1.8$ & 389.2 & 598.3 & $+209.1$ \\
2 & $L/4$ & 84.2 & 89.6 & $+5.4$  & 61.1 & 62.1 & $+1.0$ & 524.7 & 631.2 & $+106.5$ \\
2 & $L/2$ & 83.2 & 89.9 & $+6.7$  & 62.9 & 62.8 & $-0.1$ & 589.5 & 689.2 & $+ 99.7$ \\
2 & $L$   & 80.1 & 86.9 & $+6.8$  & 63.1 & 63.8 & $+0.7$ & 642.6 & 692.4 & $+ 49.8$ \\
\midrule
4 & $L/8$ & 84.8 & 89.1 & $+4.3$  & 62.2 & 63.9 & $+1.7$ & 411.8 & 624.4 & $+212.6$ \\
4 & $L/4$ & 83.5 & 89.7 & $+6.2$  & 62.7 & 65.2 & $+2.5$ & 521.9 & 678.1 & $+156.2$ \\
4 & $L/2$ & 82.8 & 90.4 & $+7.6$  & 64.0 & 66.1 & $+2.1$ & 629.2 & 692.3 & $+ 63.1$ \\
4 & $L$   & 80.4 & 91.1 & $+10.7$ & 63.1 & 64.8 & $+1.7$ & 644.3 & 678.8 & $+ 34.5$ \\
\midrule
8 & $L/4$ & 84.3 & 86.8 & $+2.5$  & 58.2 & 63.3 & $+5.1$ & 568.4 & 631.7 & $+ 63.3$ \\
8 & $L/2$ & 83.2 & 89.6 & $+6.4$  & 62.7 & 65.9 & $+3.2$ & 624.4 & 689.7 & $+ 65.3$ \\
8 & $L$   & 80.2 & 90.6 & $+10.4$ & 63.0 & 64.9 & $+1.9$ & 642.8 & 682.6 & $+ 39.8$ \\
\bottomrule
\end{tabular}
\end{table}

Principles contribute positively in 13/15 cells for POPE and in all 15 cells for MME-R, at every route count and every chain length. At fixed $\tau=2$, $\Delta$POPE ranges from $+4.3$ to $+6.8$ across the full depth range---chain shortening alone cannot explain a gain that is consistent across depths. At fixed $L/2$, $\Delta$MME-R ranges from $+11.5$ to $+63.1$---route count alone cannot explain a gain that only appears when principles are present. The two negative deltas occur exclusively at $\tau=1$, where a single principle constrains visual attention without cross-route diversity; these vanish at $\tau\ge 2$ and represent an incomplete configuration rather than a method failure.

\subsection{Grounding Robustness under Controlled Perturbations}
\label{appdx:additional:robust}

To probe whether SAP collapses when the SAM grounding pipeline is perturbed, we run three controlled perturbations spanning the full pipeline: (i) input-level Gaussian noise on the raw image, (ii) feature-level token drop on the grounded representations, and (iii) output-level item drop on the grounding summary fed to the LLM.

\begin{table}[h]
\centering
\small
\caption{SAM input-image noise. Qwen3-VL-8B; POPE/OCRVQA/MME-R. SAP stays above Instruct up to $\sigma{=}1.0$.}
\label{tab:additional:robust_input}
\setlength{\tabcolsep}{4pt}
\begin{tabular}{lccc}
\toprule
Noise $\sigma$ & POPE & OCRVQA & MME-R \\
\midrule
0.0 (clean)              & 89.9 & 62.8 & 689.9 \\
0.1                      & 89.5 & 61.2 & 671.8 \\
0.3                      & 88.1 & 60.4 & 668.4 \\
0.5                      & 87.8 & 61.2 & 661.2 \\
1.0                      & 84.4 & 58.3 & 608.8 \\
\midrule
Instruct (no grounding)  & 83.9 & 63.2 & 656.7 \\
SAM-Aug.\ LongCoT        & 14.8 & 45.2 & 328.4 \\
SAM-Aug.\ +\,SAP prompt  & 54.8 & 62.2 & 623.1 \\
\bottomrule
\end{tabular}
\end{table}

\begin{table}[h]
\centering
\small
\caption{SAM feature-token drop. Qwen3-VL-8B; POPE/OCRVQA/MME-R. Robust to $\le70\%$ drop.}
\label{tab:additional:robust_feat}
\setlength{\tabcolsep}{4pt}
\begin{tabular}{lccc}
\toprule
Drop ratio & POPE & OCRVQA & MME-R \\
\midrule
0\% (clean)              & 89.9 & 62.8 & 689.9 \\
10\%                     & 89.5 & 61.9 & 683.3 \\
30\%                     & 88.8 & 62.4 & 654.4 \\
50\%                     & 86.2 & 61.1 & 621.4 \\
70\%                     & 85.6 & 55.8 & 594.7 \\
99.9\%                   & 81.4 & 48.2 & 568.4 \\
\midrule
Instruct (no grounding)  & 83.9 & 63.2 & 656.7 \\
SAM-Aug.\ LongCoT        & 14.8 & 45.2 & 328.4 \\
SAM-Aug.\ +\,SAP prompt  & 54.8 & 62.2 & 623.1 \\
\bottomrule
\end{tabular}
\end{table}

\begin{table}[h]
\centering
\small
\caption{Grounding-summary item drop. Qwen3-VL-8B; POPE/OCRVQA/MME-R. Robust even at full drop.}
\label{tab:additional:robust_out}
\setlength{\tabcolsep}{4pt}
\begin{tabular}{lccc}
\toprule
Drop rate & POPE & OCRVQA & MME-R \\
\midrule
0\% (clean)              & 89.9 & 62.8 & 689.9 \\
25\%                     & 90.2 & 61.7 & 669.4 \\
50\%                     & 87.7 & 59.4 & 654.4 \\
75\%                     & 87.7 & 56.3 & 634.2 \\
100\% (no grounding)     & 83.8 & 62.6 & 638.4 \\
\midrule
Instruct (no grounding)  & 83.9 & 63.2 & 656.7 \\
SAM-Aug.\ LongCoT        & 14.8 & 45.2 & 328.4 \\
SAM-Aug.\ +\,SAP prompt  & 54.8 & 62.2 & 623.1 \\
\bottomrule
\end{tabular}
\end{table}

SAP remains above Instruct across all moderate perturbation levels (input $\sigma\le 1.0$, feature drop $\le 70\%$, output drop $\le 100\%$). The single exception is feature noise at $99.9\%$ (POPE 81.4), where almost-fully-randomised grounding occasionally produces locally plausible but misleading signals that partially bypass the advisory mechanism. The absence of grounding (output 100\%, POPE 83.8) is therefore \emph{safer} than fully randomised grounding (feature 99.9\%, POPE 81.4). The SAM-Augmented LongCoT baseline (POPE 14.8) confirms that catastrophic failure is not caused by grounding quality---clean grounding without the advisory prompt is already catastrophic. SAP's robustness derives from the advisory design and multi-route cross-validation operating jointly, not from grounding signal quality alone.

\subsection{Comparison with Inference-Time Search Baselines}
\label{appdx:additional:sota}

Beyond Instruct/Thinking, we compare against representative inference-time search and visual-grounding baselines under identical token budgets and the same backbone (Qwen3-VL-8B). Two of the baselines (VGR, TreeVGR) involve additional training and are listed for context only.

\begin{table}[h]
\centering
\small
\setlength{\tabcolsep}{4pt}
\renewcommand{\arraystretch}{1.3}
\caption{Inference-time search and visual-grounding baselines. Qwen3-VL-8B, 3 seeds; POPE/OCRVQA/MME-R. SAP leads while staying training-free.}
\label{tab:additional:sota}
\begin{tabular}{llccc}
\toprule
Category & Method & POPE (\%) & OCRVQA (\%) & MME-R \\
\midrule
Baseline           & Instruct~~                                  & 83.9 & 63.2 & 656.7 \\
Sequential scaling & Thinking (LongCoT)                          & 79.6 & 44.1 & 722.5 \\
Sampling           & Self-Consistency~\citep{wang2022self}       & $73.5{\pm}3.4$ & $57.8{\pm}2.1$ & $642.1{\pm}32.1$ \\
Trajectory search  & Tree of Thoughts~\citep{yao2023tree}        & --             & 62.2           & 618.2 \\
Visual decomp.\    & Multimodal-CoT~\citep{zhang2023multimodal}  & 82.5           & 58.4           & 426.7 \\
Vis.\ ground. (tr.)& VGR~\citep{wang2025vgr}                     & 87.4           & 61.2           & 710.2 \\
Vis.\ ground. (tr.)& TreeVGR~\citep{wang2026treevgr}             & 87.3           & 63.2           & 620.1 \\
\midrule
\textbf{Principle search (ours)} & \textbf{SAP, training-free}   & $\mathbf{89.9{\pm}2.4}$ & $\mathbf{62.8{\pm}1.2}$ & $\mathbf{689.9{\pm}10.8}$ \\
\bottomrule
\end{tabular}
\end{table}

SAP outperforms sampling-based and trajectory-level search baselines on perception-sensitive tasks, and is competitive with or better than the training-based visual-grounding baselines on POPE while remaining strictly training-free. This supports the claim that principle-level search with visual grounding consistency as the fitness signal addresses the modality imbalance more effectively than search procedures designed for general (text-only) reasoning.

\subsection{Variance Reports across Multiple Seeds}
\label{appdx:additional:variance}

All SAP-side reproduced numbers in Table~\ref{tab:additional:sota} are accompanied by error bars from three independent runs. Third-party reported numbers do not provide error bars and therefore appear without uncertainty intervals. SAP's standard deviation on POPE is $\pm 2.4$, on OCRVQA is $\pm 1.2$, on MME-R is $\pm 10.8$---well below the inter-method gaps reported in the main comparison.

\subsection{Smooth Sweep of Fitness Weights}
\label{appdx:additional:weightsmooth}

The four ordinal fitness components $\{c, d, e, u\}$ are combined linearly with non-negative weights. To check that the choice $w_c{=}w_d{=}w_e{=}w_u{=}1$ in the main paper is not on a sharp ridge, we sweep the trade-off between consensus weight $w_c$ and within-principle diversity weight $w_d$ along a 1D simplex (with $w_e$ fixed at $0.5$). The 16-benchmark average varies within $\sim$2 points across the interior, confirming that the operating point is robust and motivating the dataset-independent default.

\begin{table}[h]
\centering
\small
\setlength{\tabcolsep}{4pt}
\renewcommand{\arraystretch}{1.3}
\caption{Smooth $w_d$/$w_c$ sweep on Qwen3-VL-8B with $w_e{=}0.5$; 16-benchmark average. Interior values are flat.}
\label{tab:additional:weightsmooth}
\begin{tabular}{lcccccc}
\toprule
$w_d$ & 0.0 & 0.1 & 0.2 & 0.3 & 0.4 & 0.5 \\
$w_c$ & 0.5 & 0.4 & 0.3 & 0.2 & 0.1 & 0.0 \\
\midrule
avg.  & 73.0 & 73.6 & 74.8 & \cellcolor{reprobg}75.4 & 74.6 & \textbf{75.8} \\
\bottomrule
\end{tabular}
\end{table}

\subsection{Response Time on Parallel Hardware}
\label{appdx:additional:response_time}

The accompanying scatter plot reports wall-clock latency against task performance under varying numbers of parallel model instances. SAP decomposes inference into $T$ rounds of short-route generation and principle updates ($T{=}2$ by default); all routes within a round are independent and execute in parallel. Beyond a single instance, this design avoids the long token-level dependencies and intermediate-context synchronization required by LongCoT, lowering wall-clock latency while keeping total token count and total compute identical.

\input{sec/3_exp/fig/fig_response_time}
% Keep row spacing consistent with the main-paper tables.
\renewcommand{\arraystretch}{1.3}

%% file: sec/3_exp/fig/fig_response_time.tex
\begin{figure}[t]
    \centering
    \includegraphics[width=\linewidth]{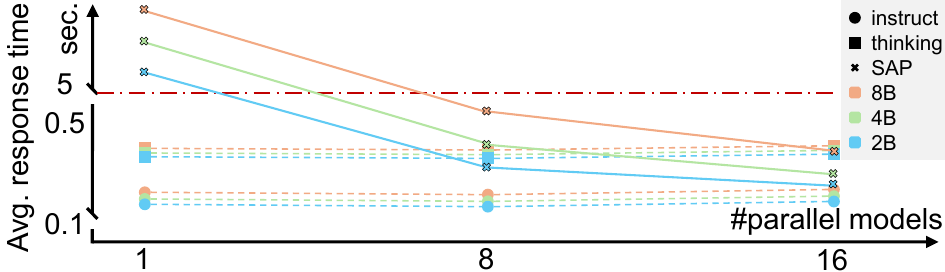}
    \caption{SAP supports inference-time speedup via parallelism.}
    \label{fig:response_time}
\end{figure}

%% file: sec/6_appdx/sec/Implementation.tex
\paragraph{SAP Overview.}
Given a question $p$ and one or multiple images $\mathcal{I}=\{I_k\}_{k=1}^{K}$, SAP performs (i) grounding, (ii) principle generation and evolutionary selection, and (iii) final aggregation. Throughout the process, the LLM only observes the \emph{original images} and a compact \emph{textual grounding summary}; object lines and crops are used only for fitness evaluation and are never fed back to the LLM. All assessment signals are discretized into three levels: \texttt{low}/\texttt{medium}/\texttt{high}.

\paragraph{Step 1: Multi-image grounding.}
For each image $I_k$, SAP calls a grounding service (e.g., SAM-based AMG) to obtain a set of grounded objects $\mathcal{O}_k=\{o_{k,j}\}_{j=1}^{M_k}$, each with a label $\ell_{k,j}$ (obtained via an external captioner/labeler), and optional mask/bbox/score metadata.
SAP converts $\{\mathcal{O}_k\}_{k=1}^{K}$ into a compact grounding summary string $g$ that only describes counts and image sizes (optionally plus other lightweight statistics). SAP also builds a label map so that any object reference can be canonicalized as \texttt{img\#k\_obj\#j($\ell_{k,j}$)}.

\paragraph{Step 2: Initialize a principle population.}
SAP samples an initial population of principles $\{x_i\}_{i=1}^{\mu+\lambda}$ by a single LLM call conditioned on $(p,g,\mathcal{I})$. Each $x_i$ is a short textual instruction that specifies a reasoning strategy (e.g., ``verify with grounded objects before answering'').

\paragraph{Step 3: One-call multi-route reasoning per principle.}
For each generation, SAP evaluates each principle $x_i$ with exactly one LLM call. The LLM is provided with $(p, g, \mathcal{I}, x_i)$ and is asked to internally generate $\tau$ distinct reasoning routes under the same principle:
\[
\{(r_{i,j}, s_{i,j})\}_{j=1}^{\tau}.
\]
Each route produces a candidate \texttt{final\_answer}, a short list of reasons, an uncertainty level $u_{i,j}\in\{\texttt{low},\texttt{medium},\texttt{high}\}$, and a list of evidence references. The LLM also outputs a within-principle diversity level $d_i\in\{\texttt{low},\texttt{medium},\texttt{high}\}$ summarizing how different the $\tau$ routes are.

\paragraph{Step 4: Representative answer and global consensus.}
SAP selects a representative route per principle (typically preferring lower uncertainty and non-empty answers) and obtains a representative answer $y_i^\star$.
SAP then computes a global consensus answer $\bar{y}$ by majority vote over $\{y_i^\star\}$ and converts the winning fraction into a discrete consensus level $c\in\{\texttt{low},\texttt{medium},\texttt{high}\}$. The semantic-equivalence judgment $\sim_\varphi$ used below is performed by the same VLM backbone in zero-shot judge mode (Qwen3-VL-8B in our default configuration), consistent with treating the lightweight external judge of Eq.~\eqref{eq:sap_fitness} as a separate inference-time call rather than a separately trained model. In addition, SAP assigns each principle a discrete match-to-consensus level
\[
c_i=
\begin{cases}
\texttt{high}, & y_i^\star=\bar{y},\\
\texttt{medium}, & y^{*}_{i}\sim_{\varphi}\bar{y}\\
\texttt{low}, & \text{otherwise}.
\end{cases}
\]
This design enforces that principles are rewarded only if their representatives' answers align with the population-level agreement. $y^{*}\sim_{\varphi} \bar{y}$ means $y^{*}$ is judged as the same meaning of $\bar{y}$ by language model $\varphi$.

\paragraph{Step 5: Evidence validation and discretization.}
SAP canonicalizes each cited evidence item into the format \texttt{img\#k\_obj\#j($\ell_{k,j}$)} using the label map.
It then evaluates whether cited object indices are valid grounding outputs (valid vs.\ invalid references). The ratio of valid references is discretized into an evidence level $e_i\in\{\texttt{low},\texttt{medium},\texttt{high}\}$, with empty evidence treated as \texttt{low}. This evidence level is used as a lightweight grounding-consistency signal; it never requires the LLM to see object crops.

\paragraph{Step 6: Discrete-level fitness and selection.}
For each principle $x_i$, SAP computes a scalar fitness score from four discrete signals: consensus match $c_i$, within-principle diversity $d_i$, representative uncertainty $u_i$ (from the representative route), and evidence level $e_i$.
Internally, SAP maps \texttt{low}/\texttt{medium}/\texttt{high} to ordinal scores and combines them with integer weights:
\[
f_i \leftarrow
w_c\,\mathrm{Score}(c_i) +
w_d\,\mathrm{Score}(d_i) +
w_e\,\mathrm{Score}(e_i) -
w_u\,\mathrm{Pen}(u_i),
\]
where $\mathrm{Score}(\texttt{low},\texttt{medium},\texttt{high})=(0,1,2)$ and $\mathrm{Pen}(\texttt{low},\texttt{medium},\texttt{high})=(0,1,2)$.
SAP selects the top-$\mu$ principles as elites and (if more generations remain) asks the LLM to propose $\lambda$ new principles conditioned on the elites and their fitness, forming the next $(\mu+\lambda)$ population.

\paragraph{Step 7: Final aggregation.}
After the last generation, SAP produces the final answer by aggregating the evaluated principles.
Optionally, it issues one final LLM call conditioned on $(p,g)$ and the list of representative summaries to synthesize a final answer with brief reasons and a discrete uncertainty level. If the aggregator fails, SAP falls back to the representative answer from the global-consensus group with conservative uncertainty.

%% file: sec/2_mtd/sec/framework.tex
Building on the above formulation, SAP approximates the principle-level optimization objective through population-based evolutionary search. This procedure scales inference by iteratively improving a population of reasoning principles with respect to the utility function $f$.

The overall procedure consists of initialization and iterative evolution. As shown in Alg.~\ref{alg:init}, an initial population of principles $\{x_i\}_{i \in [\mu+\lambda]}$ is sampled conditioned on the prompt $p$. For each principle $x_i$, the language model generates multiple reasoning traces and summaries,
\[
(r_{i,j}, s_{i,j}) \sim \varphi(r, s \mid x_i, p), \quad j \in [\tau],
\]
which instantiate the reasoning route $\rho(x_i)$ induced by the principle.

Given these outputs, each principle is evaluated by a task-specific utility metric, producing a fitness score
\[
f_i = f(\{r_{i,j}, s_{i,j}\}),
\]
which estimates the effectiveness of the corresponding reasoning route. Based on these fitness scores, SAP performs population evolution as described in Alg.~\ref{alg:evolve}. Specifically, a $(\mu+\lambda)$ selection scheme is adopted: the top-$\mu$ principles
\[
\{x_m^\mu\}_{m \in [\mu]} = \mathrm{Top}_{\mu}\big(\{(x_i, f_i)\}_{i \in [\mu+\lambda]}\big)
\]
are retained as elites, while a new set of $\lambda$ principles is generated by sampling conditioned on the elites,
\[
x_n^\lambda \sim \varphi(x \mid \{x_m^\mu\}_{m\in[\mu]}, p), \quad n \in [\lambda].
\]
The elites and newly generated principles are combined to form the updated population
\[
\{x_i^{+}\}_{i \in [\mu+\lambda]} = \{x_m^\mu\}_{m \in [\mu]} \cup \{x_n^\lambda\}_{n \in [\lambda]}.
\]

By iteratively applying population evolution, SAP provides a population-based approximation to the optimization problem in Eq.~(\ref{equ:formulation}), enabling robust inference-time scaling under discrete generation and noisy feedback.

%% file: sec/6_appdx/sec/Benchmarks.tex
\paragraph{Benchmarks.}
We evaluate SAP on a diverse set of vision--language benchmarks spanning robustness, general reasoning, scientific understanding, OCR-centric reasoning, document analysis, and real-world generalization.

\textbf{POPE}~\cite{li2023evaluating} evaluates object hallucination and existence verification by constructing balanced positive and negative samples.
It directly measures whether models can correctly ground answers in visual evidence rather than prior bias.

\textbf{MME}~\cite{fu2025mme} assesses multimodal robustness by separating perception and reasoning abilities.
The reasoning subset emphasizes logical inference beyond direct visual recognition, making it suitable for evaluating inference-time reasoning strategies.

\textbf{MMBench}~\cite{liu2024mmbench} is a large-scale multiple-choice benchmark for general multimodal understanding, covering object recognition, attribute reasoning, spatial relations, and commonsense inference.
It serves as a standard indicator of overall VLM capability.

\textbf{SEEDBench}~\cite{li2023seed} evaluates general visual reasoning and understanding across perception, logic, and commonsense tasks, emphasizing consistency across heterogeneous question types.

\textbf{MMMU}~\cite{yue2024mmmu} targets college-level multimodal reasoning in mathematics, physics, chemistry, and engineering.
Many questions involve diagrams or multiple images, requiring structured reasoning and cross-image integration.

\textbf{ScienceQA}~\cite{lu2022scienceqa} evaluates scientific question answering with visual contexts, combining diagrams with domain knowledge and multi-step reasoning.

\textbf{TextVQA}~\cite{singh2019towards} measures scene text understanding, where correct answers depend on reading and reasoning over text embedded in natural images.

\textbf{OCRVQA}~\cite{mishra2019ocrvqa} and \textbf{OCRBench}~\cite{liu2024ocrbench,fu2024ocrbench} further stress OCR-centric reasoning under complex layouts and challenging visual conditions, focusing on accuracy and robustness of text grounding.

\textbf{ChartQA}~\cite{masry2022chartqa} evaluates reasoning over charts and plots, requiring numerical comparison and interpretation grounded in visualized data.

\textbf{DocVQA}~\cite{mathew2021docvqa} and \textbf{InfoVQA}~\cite{mathew2022infographicvqa} focus on document understanding, where models must reason over scanned documents, tables, and forms with dense textual content.

\textbf{AI2D}~\cite{kembhavi2016diagram} tests diagram understanding in science education scenarios, emphasizing relational and structural reasoning over visual components.

\textbf{BLINK}~\cite{fu2024blink} evaluates multimodal entity linking, requiring alignment between visual entities and background knowledge.

\textbf{RealWorldQA} assesses open-world robustness using real-life images and unconstrained questions, testing generalization beyond curated benchmarks~\cite{realworldqa}. \url{https://huggingface.co/datasets/visheratin/realworldqa}

%% file: sec/6_appdx/sec/Analysis_Framework.tex
\subsection{Notation and Setup}
Let $\mathcal{X}$ denote the (discrete) space of reasoning principles. The $(\mu+\lambda)$ evolutionary procedure maintains a population
\[
P_t = \{x^{(t)}_i\}_{i=1}^{N}, \qquad N=\mu+\lambda,
\]
at iteration $t\in\{0,1,\dots,T\}$. Each principle $x\in\mathcal{X}$ induces a set of reasoning outcomes (trajectories and summaries) via one model call, and is assigned a \emph{fitness} value $f(x)\in\mathbb{R}$ through the discrete assessment and ordinal mapping described in Eq.~\eqref{eq:sap_fitness}. Define the best fitness in the population at iteration $t$ as
\[
F_t = \max_{x\in P_t} f(x).
\]

\paragraph{Elites and offspring.}
At iteration $t$, the elite set $E_t\subseteq P_t$ consists of the top-$\mu$ principles in $P_t$ ranked by fitness. The offspring multiset
\(
O_t=\{X^{(t)}_n\}_{n=1}^{\lambda}
\)
is generated by sampling from a conditional distribution induced by the LLM-based proposer:
\[
X^{(t)}_n \sim \varphi_x(\cdot \mid p, E_t), \qquad n\in[\lambda],
\]
and the next population is formed as
\[
P_{t+1}=E_t \cup O_t.
\]
We emphasize that the discrete evaluation in SAP is used only to produce ordinal levels and an internal ranking/selection signal; no continuous confidence calibration is assumed.

\subsection{Assumptions on Sampling and Dependence}

\begin{assumption}[Conditional i.i.d.\ offspring sampling]
\label{ass:iid}
Conditioned on the history $\mathcal{H}_t$ up to iteration $t$ (which includes $p$, all past populations, and thus $E_t$), the offspring samples
\(
X^{(t)}_1,\dots,X^{(t)}_\lambda
\)
are conditionally independent and identically distributed according to $\varphi_x(\cdot\mid p,E_t)$.
\end{assumption}

Assumption~\ref{ass:iid} is satisfied when each offspring is generated by an independent call to the proposer (or independent decoding randomness) conditioned on the same $(p,E_t)$. We use Assumption~\ref{ass:iid} only to derive the search-space \emph{coverage} bound (\S below): once the elite set $E_t$ has stabilized in the late iterations of the loop, successive offspring populations sample from a near-stationary conditional distribution $\varphi_x(\cdot\mid p, E_t)$, and the resulting offspring are approximately i.i.d.\ at the level of \emph{semantic principle classes} (the granularity at which the coverage bound is stated), even though their token-level realizations are produced by autoregressive decoding. If implementation introduces stronger dependence, the probability bounds below can be extended under standard negative dependence conditions; we state the i.i.d.\ form for clarity.

\subsection{Optimization Guarantees}

\subsubsection{Monotone best-fitness (no regression)}
\begin{theorem}[Monotone best-fitness under $(\mu+\lambda)$ selection]
\label{thm:mono_app}
For all iterations $t\ge 0$, the best fitness is non-decreasing:
\[
F_{t+1} \ge F_t.
\]
\end{theorem}

\begin{proof}
By definition, $E_t$ contains the top-$\mu$ elements of $P_t$ ordered by $f(\cdot)$, and therefore contains at least one maximizer of $f$ over $P_t$. Hence
\[
\max_{x\in E_t} f(x)=\max_{x\in P_t} f(x)=F_t.
\]
Since $P_{t+1}=E_t\cup O_t$ and $E_t\subseteq P_{t+1}$,
\[
F_{t+1}=\max_{x\in P_{t+1}} f(x)\ge \max_{x\in E_t} f(x)=F_t.
\]
\end{proof}

\subsubsection{One-step improvement probability}
While Theorem~\ref{thm:mono_app} shows stability (no regression), the \emph{optimization speed} is governed by how likely the procedure produces strictly better principles.

\paragraph{Improvement set.}
Define the set of principles that strictly improve upon the current best:
\[
\mathcal{A}_t = \{x\in\mathcal{X}:\ f(x)>F_t\}.
\]

\begin{assumption}[Per-iteration improvement mass]
\label{ass:qt}
There exists $q_t\in[0,1]$ such that conditioned on $\mathcal{H}_t$,
\[
\Pr\!\left(X^{(t)}_n \in \mathcal{A}_t \mid \mathcal{H}_t\right)\ge q_t
\quad \text{for each } n\in[\lambda].
\]
\end{assumption}

Assumption~\ref{ass:qt} does not require knowing $\mathcal{A}_t$ explicitly; it only postulates that the proposer distribution retains at least $q_t$ probability mass on strictly improving principles.

\begin{theorem}[One-step improvement bound]
\label{thm:improve_app}
Under Assumptions~\ref{ass:iid} and~\ref{ass:qt}, the probability of a strict improvement at iteration $t$ satisfies
\[
\Pr(F_{t+1}>F_t \mid \mathcal{H}_t)\ge 1-(1-q_t)^\lambda.
\]
Consequently,
\[
\Pr(F_{t+1}>F_t)\ge \mathbb{E}\!\left[\,1-(1-q_t)^\lambda\,\right].
\]
\end{theorem}

\begin{proof}
A strict improvement occurs if at least one offspring lies in $\mathcal{A}_t$:
\[
\{F_{t+1}>F_t\} \supseteq \bigcup_{n=1}^{\lambda}\{X^{(t)}_n \in \mathcal{A}_t\}.
\]
Equivalently, no improvement implies none of the offspring hits $\mathcal{A}_t$:
\[
\{F_{t+1}=F_t\} \subseteq \bigcap_{n=1}^{\lambda}\{X^{(t)}_n \notin \mathcal{A}_t\}.
\]
Conditioning on $\mathcal{H}_t$ and using conditional independence (Assumption~\ref{ass:iid}),
\begin{align*}
\Pr(F_{t+1}=F_t \mid \mathcal{H}_t)
&\le \Pr\!\left(\bigcap_{n=1}^{\lambda}\{X^{(t)}_n \notin \mathcal{A}_t\}\ \middle|\ \mathcal{H}_t\right) \\
&= \prod_{n=1}^{\lambda}\Pr(X^{(t)}_n \notin \mathcal{A}_t \mid \mathcal{H}_t) \\
&\le (1-q_t)^\lambda,
\end{align*}
where the last inequality follows from Assumption~\ref{ass:qt}. Taking complements gives
\[
\Pr(F_{t+1}>F_t \mid \mathcal{H}_t)
=1-\Pr(F_{t+1}=F_t \mid \mathcal{H}_t)
\ge 1-(1-q_t)^\lambda.
\]
Finally, removing the conditioning yields the stated unconditional bound by taking expectation over $\mathcal{H}_t$.
\end{proof}

\begin{remark}[Small-$q_t$ regime and the role of $\lambda$]
\label{rem:smallq}
When $q_t$ is small, the improvement probability satisfies
\[
1-(1-q_t)^\lambda \approx \lambda q_t,
\]
showing that increasing $\lambda$ approximately linearly amplifies the chance of obtaining a strictly better principle in each iteration.
\end{remark}

\subsubsection{Multiple iterations: a (loose) bound on improvement within $T$ steps}
For completeness, we provide a standard union-style lower bound on obtaining at least one strict improvement over a horizon of $T$ iterations.

\begin{corollary}[At least one improvement within $T$ iterations (lower bound)]
\label{cor:improve_T}
Under Assumptions~\ref{ass:iid} and~\ref{ass:qt}, and assuming $q_t\ge q_{\min}$ for $t=0,\dots,T-1$,
\[
\Pr(\exists\, t\in\{0,\dots,T-1\}: F_{t+1}>F_t) \ge 1-(1-q_{\min})^{\lambda T}.
\]
\end{corollary}

\begin{proof}
Condition on each iteration and apply Theorem~\ref{thm:improve_app}:
\[
\Pr(F_{t+1}=F_t\mid \mathcal{H}_t)\le (1-q_t)^\lambda \le (1-q_{\min})^\lambda.
\]
Assuming conditional independence across iterations is not necessary for an upper bound on the probability of \emph{no improvements}: by iterated conditioning,
\[
\Pr(\text{no improvements in } T \text{ iterations})
\le \prod_{t=0}^{T-1}(1-q_{\min})^\lambda
=(1-q_{\min})^{\lambda T}.
\]
Taking complements yields the claim.
\end{proof}

The above corollary is intentionally conservative; it is included to illustrate scaling with $\lambda T$ and to align with coverage-style arguments below.

\subsection{Generalization as Coverage in the Principle Space}

\subsubsection{Coverage model}
We use a coverage-based notion of generalization: larger inference budgets explore more candidate principles, increasing the probability of encountering effective principles that induce reliable reasoning behaviors beyond the initial population.

\begin{definition}[Effective principle set]
\label{def:good}
Let $\mathcal{X}_{\mathrm{good}}\subseteq\mathcal{X}$ denote the (unknown) set of effective principles under the task distribution of interest (e.g., principles that produce stable, grounded, and reliable outputs according to SAP's discrete assessments).
\end{definition}

\begin{assumption}[Per-sample coverage lower bound]
\label{ass:qgood}
There exists $q_{\mathrm{good}}\in[0,1]$ such that for every iteration $t$ and every offspring sample $X^{(t)}_n$,
\[
\Pr\!\left(X^{(t)}_n\in \mathcal{X}_{\mathrm{good}} \mid \mathcal{H}_t\right)\ge q_{\mathrm{good}}.
\]
\end{assumption}

Assumption~\ref{ass:qgood} states that the proposer retains a nontrivial probability of sampling effective principles. This is the standard requirement for coverage guarantees in randomized search.

\subsubsection{Discovery probability within $T$ iterations}
\begin{proposition}[Coverage increases with $\lambda T$]
\label{prop:coverage_app}
Under Assumptions~\ref{ass:iid} and~\ref{ass:qgood}, the probability that at least one effective principle is sampled within $T$ iterations is lower bounded by
\[
\Pr\!\left(\exists\, t\in\{0,\dots,T-1\},\ \exists\, n\in[\lambda] \text{ s.t. } X^{(t)}_n\in\mathcal{X}_{\mathrm{good}}\right)
\ge 1-(1-q_{\mathrm{good}})^{\lambda T}.
\]
\end{proposition}

\begin{proof}
Define the event that no effective principle is sampled in iteration $t$:
\[
\mathcal{E}_t := \bigcap_{n=1}^{\lambda}\{X^{(t)}_n\notin \mathcal{X}_{\mathrm{good}}\}.
\]
Conditioned on $\mathcal{H}_t$ and using conditional independence (Assumption~\ref{ass:iid}),
\[
\Pr(\mathcal{E}_t\mid \mathcal{H}_t)
=\prod_{n=1}^{\lambda}\Pr(X^{(t)}_n\notin\mathcal{X}_{\mathrm{good}}\mid \mathcal{H}_t)
\le (1-q_{\mathrm{good}})^\lambda,
\]
where the inequality follows from Assumption~\ref{ass:qgood}.
Now consider the event of no discovery over $T$ iterations: $\bigcap_{t=0}^{T-1}\mathcal{E}_t$. By iterated conditioning,
\[
\Pr\!\left(\bigcap_{t=0}^{T-1}\mathcal{E}_t\right)
= \mathbb{E}\!\left[\prod_{t=0}^{T-1} \Pr(\mathcal{E}_t \mid \mathcal{H}_t)\right]
\le \prod_{t=0}^{T-1}(1-q_{\mathrm{good}})^\lambda
= (1-q_{\mathrm{good}})^{\lambda T}.
\]
Taking complements yields
\[
\Pr\!\left(\bigcup_{t=0}^{T-1}\mathcal{E}_t^{c}\right)
\ge 1-(1-q_{\mathrm{good}})^{\lambda T},
\]
which is exactly the probability of sampling at least one effective principle within $T$ iterations.
\end{proof}

\begin{remark}[Scaling in $\lambda$ and $T$]
\label{rem:scaling}
The bound $1-(1-q_{\mathrm{good}})^{\lambda T}$ is monotone increasing in both $\lambda$ and $T$. For small $q_{\mathrm{good}}$, it satisfies
\[
1-(1-q_{\mathrm{good}})^{\lambda T}\approx \lambda T q_{\mathrm{good}},
\]
making explicit that increasing the exploration width ($\lambda$) and/or the number of iterations ($T$) improves coverage of the principle space.
\end{remark}

\subsection{Connecting Optimization and Coverage to Inference-Time Scaling}
The results above jointly support the inference-time scaling interpretation of SAP/PER:
elite selection ensures optimization stability (Theorem~\ref{thm:mono_app}), while increasing inference budget via larger $\lambda$ and/or $T$ improves both the per-iteration chance of strict improvement (Theorem~\ref{thm:improve_app}) and the probability of discovering effective principles through coverage (Proposition~\ref{prop:coverage_app}). These properties match the operational definition of scaling in our setting: allocating additional compute to explore and select among alternative reasoning principles and routes, rather than extending a single fixed trajectory.

%% file: sec/6_appdx/sec/Template.tex
\paragraph{System Prompt Template Design.}
A central objective of SAP is to prevent inference-time reasoning from drifting toward text-only heuristics and to ensure that reasoning remains grounded in visual evidence throughout long generation horizons.
This objective is realized through a carefully designed family of user prompt templates that are shared across all stages of SAP, with stage-specific constraints. \emph{Different VLMs should correspond to different templates.}

\textbf{Motivation.}
We identify two recurring failure modes in long-horizon vision–language reasoning.
First, visual information is typically provided only once at the beginning of generation, whereas textual reasoning states (e.g., early visual summaries) are updated autoregressively.
As generation proceeds, reasoning therefore becomes increasingly text-dominated, allowing early visual grounding errors to compound over time.
Second, auxiliary textual signals—such as grounding summaries or evaluator feedback—are often noisy, incomplete, or heuristic; yet, they are easily over-trusted by the model as factual evidence.
Together, these effects cause models to rely excessively on intermediate text states instead of re-consulting the original image(s), leading to unstable inference-time scaling.

\textbf{Design Principle.}
To address these issues, SAP enforces a strict evidence hierarchy in all user prompts.
The original image(s) are explicitly designated as the primary and most reliable source of information.
Textual grounding summaries are framed as optional references that may be missing, incomplete, or incorrect.
Their role is limited to suggesting \emph{where to look} in the image, rather than \emph{what to believe}.
Whenever textual cues conflict with visual observations, the model is instructed to re-examine the image(s) and respond conservatively if ambiguity remains.
This principle is shared across all SAP stages to counteract text-dominated autoregressive drift.

\textbf{Template Overview.}
SAP employs four user prompt templates corresponding to different inference-time stages:
(\emph{i}) principal initialization,
(\emph{ii}) principal evolution,
(\emph{iii}) principle-guided multi-route reasoning,
and
(\emph{iv}) final answer aggregation.
Despite their different roles, all templates explicitly reinforce visual dominance, uncertainty awareness, and conservative decision-making under noisy feedback.

\paragraph{(1) Principle Initialization Template.}
The goal of principle initialization is to generate a diverse set of high-level reasoning principles before producing concrete answers.
The prompt instructs the model to focus on \emph{how} to reason from visual evidence rather than encoding task-specific conclusions.

\begin{lstlisting}[basicstyle=\ttfamily\small]
Generate several high-level reasoning principles for answering the visual question.
The image(s) are the primary source of truth.
The grounding summary is optional and may be incomplete or incorrect.
Each principle should describe a general way to reason from visual evidence,
not a specific answer.
Avoid assuming details that cannot be directly verified from the image(s).
\end{lstlisting}

This template encourages diversity in reasoning strategies while ensuring that all principles remain anchored to observable visual cues.

\paragraph{(2) Principle Evolution Template.}
During evolutionary refinement, principles are updated based on their observed utility.
The evolution prompt presents prior principles and their relative performance while explicitly discouraging overfitting to noisy or short-term feedback.

\begin{lstlisting}[basicstyle=\ttfamily\small]
Refine the reasoning principles for visual question answering.
The image(s) remain the most reliable evidence.
Fitness signals and summaries may be noisy and should be used cautiously.
Preserve strengths of effective principles, discard misleading ones,
and propose new principles that better encourage verification
against visual evidence.
Do not assume that high-scoring principles are always correct.
\end{lstlisting}

This design ensures that evolutionary updates favor principles that improve visual grounding rather than exploit accidental correlations in evaluator signals.

\paragraph{(3) Principle-Guided Multi-route Reasoning Template.}
Given a fixed principle, SAP requires the model to generate multiple distinct reasoning routes in parallel.
Each route must independently justify its conclusion using visual evidence and explicitly report uncertainty when the evidence is weak.

\begin{lstlisting}[basicstyle=\ttfamily\small]
Answer the visual question using the image(s)
and the single active reasoning principle below.
Generate multiple distinct reasoning routes under the same principle.
For each route:
- Base the reasoning on observable visual evidence.
- Cite visual cues when possible.
- Report uncertainty if evidence is weak or ambiguous.
The grounding summary is only a reference and may be incorrect.
If unsure, do not guess.
\end{lstlisting}

By construction, this template discourages blind reliance on early textual summaries and promotes repeated visual consultation across alternative reasoning routes.

\paragraph{(4) Final Aggregation Template.}
The final aggregation stage synthesizes candidate answers produced under different principles.
The prompt emphasizes conservative selection based on visual support rather than textual agreement or majority voting.

\begin{lstlisting}[basicstyle=\ttfamily\small]
Synthesize a final answer from multiple candidate answers.
Treat the image(s) as the authoritative evidence.
Candidate answers may be inconsistent or incorrect.
Prefer answers that are better supported by visual evidence.
Express uncertainty if no answer is clearly justified.
Do not introduce new assumptions beyond what can be verified
from the image(s).
\end{lstlisting}

\paragraph{Implicit Template Constraints.}
Beyond stage-specific instructions, SAP enforces several implicit constraints that are shared across all prompt templates.
These constraints are implemented purely at the prompt level and shape inference-time behavior without introducing additional supervision.

\textbf{Evidence Hierarchy Block.}
All SAP prompts explicitly prioritize visual evidence over any textual signals.

\begin{lstlisting}[basicstyle=\ttfamily\small]
[Evidence Hierarchy]
- Image(s) are the most reliable evidence.
- Textual summaries and feedback are auxiliary and may be noisy.
- Text must never override visual observation.
- When uncertainty remains, respond conservatively.
\end{lstlisting}

\textbf{Grounding-as-Reference Rule.}
Grounding summaries are treated strictly as weak references.

\begin{lstlisting}[basicstyle=\ttfamily\small]
[Grounding Reference Rule]
- Grounding summaries may be incomplete or incorrect.
- Use them only to guide where to look in the image(s).
- Do not assume grounded attributes unless visually verified.
- If conflicts arise, trust the image(s).
\end{lstlisting}

\textbf{Strict Output Rule.}
To support reliable evaluation and comparison, SAP enforces a strict structured-output constraint.

\begin{lstlisting}[basicstyle=\ttfamily\small]
[Strict Output Rule]
- Follow the specified JSON schema exactly.
- Do not add extra fields or free-form explanations.
- Use predefined discrete levels for uncertainty and evaluation.
- If information is missing, return a valid placeholder.
\end{lstlisting}

\textbf{Effect.}
Together, these prompt templates and shared constraints form a unified inference-time control mechanism that reallocates computation toward repeated validation against visual evidence.
SAP introduces no new knowledge, external supervision, or task-specific heuristics.
Instead, it reshapes how existing model capacity is used during inference, enabling stable and effective inference-time scaling under discrete generation and noisy multimodal feedback.

%% file: neurips/checklist_camera_ready.tex
\section*{NeurIPS Paper Checklist}

\begin{enumerate}

\item {\bf Claims}
    \item[] Question: Do the main claims made in the abstract and introduction accurately reflect the paper's contributions and scope?
    \item[] Answer: \answerYes{}
    \item[] Justification: The abstract and Section~1 list three concrete contributions --- characterising text-dominated drift in long VLM reasoning, proposing the Saliency-Aware Principle Selection (SAP) framework, and demonstrating parallelizable inference-time scaling --- all of which are directly supported by Section~4 (16-benchmark comparison) and Appendix~\ref{appdx:additional} (interaction grid, robustness, baselines).

\item {\bf Limitations}
    \item[] Question: Does the paper discuss the limitations of the work performed by the authors?
    \item[] Answer: \answerYes{}
    \item[] Justification: Limitations are explicitly discussed in Appendix~A (More Discussion): SAP's amortized efficiency relies on parallel inference settings, performance depends on prompt-template quality for principle generation, and reasoning-intensive tasks (\textit{e.g.}, MathVerse) require longer per-route budgets to fully match LongCoT (Appendix~\ref{appdx:additional:mathbench}). The grounding-tool dependency and its failure modes are quantified under controlled perturbation in Appendix~\ref{appdx:additional:robust}.

\item {\bf Theory assumptions and proofs}
    \item[] Question: For each theoretical result, does the paper provide the full set of assumptions and a complete (and correct) proof?
    \item[] Answer: \answerYes{}
    \item[] Justification: The single theoretical result (Theorem~\ref{thm:opt_gen}, on optimization stability and coverage of the $(\mu+\lambda)$ evolutionary procedure) is stated in Section~3 (Framework Analysis). All assumptions (Assumption~\ref{ass:iid}, \ref{ass:qt}, \ref{ass:qgood}) are stated formally in Appendix~\ref{appdx:framework_analysis}, and complete proofs of the monotone-fitness, one-step improvement, and coverage bounds appear in the same appendix.

    \item {\bf Experimental result reproducibility}
    \item[] Question: Does the paper fully disclose all the information needed to reproduce the main experimental results of the paper to the extent that it affects the main claims and/or conclusions of the paper (regardless of whether the code and data are provided or not)?
    \item[] Answer: \answerYes{}
    \item[] Justification: SAP is training-free; all reproducibility information is therefore in the inference recipe. The full pipeline (grounding, principle initialization, multi-route inference, fitness, evolution, aggregation) is described step-by-step in Appendix~\ref{appdx:implementation} (Implementation), with the four prompt templates given verbatim in the Template Design appendix. Hyperparameters ($\mu=\lambda=2$, $\tau=2$, $T=2$, $w_c{=}w_d{=}w_e{=}w_u{=}1$) and decoding settings are listed in the Experiment Setups subsection of Section~4. The 16 benchmarks are evaluated through the public VLMEvalKit harness with no per-dataset tuning.

\item {\bf Open access to data and code}
    \item[] Question: Does the paper provide open access to the data and code, with sufficient instructions to faithfully reproduce the main experimental results, as described in supplemental material?
    \item[] Answer: \answerYes{}
    \item[] Justification: We provide the core logic of SAP (principle initialization, multi-route inference, fitness scoring, $(\mu+\lambda)$ evolution, aggregation) in the camera-ready supplementary material. Benchmark harness wrappers and dataset loaders are provided by the public VLMEvalKit package rather than duplicated. The implementation is intentionally lightweight: the algorithmic core is given step-by-step in Appendix~\ref{appdx:implementation}, the four prompt templates are reproduced verbatim in the Template Design appendix, and all hyperparameters in Section~4 are dataset-independent, so SAP can be applied to any frozen VLM through its standard inference API. Evaluation uses the public VLMEvalKit harness with no per-dataset tuning.

\item {\bf Experimental setting/details}
    \item[] Question: Does the paper specify all the training and test details (e.g., data splits, hyperparameters, how they were chosen, type of optimizer) necessary to understand the results?
    \item[] Answer: \answerYes{}
    \item[] Justification: SAP is training-free, so there are no optimizer/learning-rate choices to report. Test-time configuration --- backbones (Qwen3-VL-8B and additional 2B/4B/30B/InternVL3.5-8B/DeepSeek-VL2 for ablations), evaluation splits via VLMEvalKit, decoding budget, and all SAP hyperparameters --- is given in the Experiment Setups subsection of Section~5 and Appendix~B.

\item {\bf Experiment statistical significance}
    \item[] Question: Does the paper report error bars suitably and correctly defined or other appropriate information about the statistical significance of the experiments?
    \item[] Answer: \answerYes{}
    \item[] Justification: SAP-side reproduced numbers in the inference-time-search comparison table (Appendix~\ref{appdx:additional:sota}) carry $\pm$ standard deviation across three independent runs (seeds 2025, 2026, 42). The marginal-distribution analyses in Appendix~\ref{appdx:additional:adherence} and Appendix~\ref{appdx:additional:grid} also report $\pm$ values. Third-party reported numbers are listed without error bars because the original references do not provide them.

\item {\bf Experiments compute resources}
    \item[] Question: For each experiment, does the paper provide sufficient information on the computer resources (type of compute workers, memory, time of execution) needed to reproduce the experiments?
    \item[] Answer: \answerYes{}
    \item[] Justification: SAP is training-free; all GPU work runs on a single NVIDIA A100. CPU resources come from a cloud-hosted virtual machine and the underlying physical CPU specifications are not exposed by the provider. Software environment versions (Python, PyTorch, VLMEvalKit, etc.) are pinned in the \texttt{requirements} file shipped with the supplementary material, which together with the A100 GPU is sufficient to reproduce the runs. The response-time analysis in Section~4.3 reports relative wall-clock latency between SAP and LongCoT under this hardware; further wall-clock numbers (per-route generation time and parallel-instance scaling) are reported in Appendix~\ref{appdx:additional:response_time}.

\item {\bf Code of ethics}
    \item[] Question: Does the research conducted in the paper conform, in every respect, with the NeurIPS Code of Ethics \url{https://neurips.cc/public/EthicsGuidelines}?
    \item[] Answer: \answerYes{}
    \item[] Justification: The authors have read the NeurIPS Code of Ethics and confirm full conformance. The work introduces no new datasets, no human-subject data, and no high-risk model release; SAP is a training-free inference procedure layered on publicly available pre-trained VLMs.

\item {\bf Broader impacts}
    \item[] Question: Does the paper discuss both potential positive societal impacts and negative societal impacts of the work performed?
    \item[] Answer: \answerYes{}
    \item[] Justification: An Impact Statement is included before the acknowledgments. It notes that improved visual grounding may make VLM outputs more reliable, while SAP's multi-route inference increases total model computation relative to a single short response and may therefore increase energy use and its associated environmental footprint. Parallel execution reduces wall-clock latency but does not remove this additional computation.

\item {\bf Safeguards}
    \item[] Question: Does the paper describe safeguards that have been put in place for responsible release of data or models that have a high risk for misuse (e.g., pre-trained language models, image generators, or scraped datasets)?
    \item[] Answer: \answerNA{}
    \item[] Justification: The paper releases no new pre-trained models, no scraped datasets, and no image generators. SAP is a training-free inference procedure operating on top of existing publicly available VLMs (Qwen3-VL, InternVL3.5, DeepSeek-VL2) and a frozen grounding tool (SAM); responsible release is governed by the upstream providers of these assets.

\item {\bf Licenses for existing assets}
    \item[] Question: Are the creators or original owners of assets (e.g., code, data, models), used in the paper, properly credited and are the license and terms of use explicitly mentioned and properly respected?
    \item[] Answer: \answerYes{}
    \item[] Justification: All upstream assets are credited by citation: backbones (Qwen3-VL~\citep{bai2025qwen3vl}, InternVL3.5~\citep{wang2025internvl3}, DeepSeek-VL2~\citep{wu2024deepseek}), the grounding tool (SAM~\citep{kirillov2023segment}), the evaluation harness (VLMEvalKit~\citep{duan2024vlmevalkit}), and every benchmark in the 16-benchmark suite (Appendix~D). Appendix~A states the applicable licenses for the models, model code, grounding tool, and evaluation harness; benchmark data are used under the terms of their original distributions and are not redistributed.

\item {\bf New assets}
    \item[] Question: Are new assets introduced in the paper well documented and is the documentation provided alongside the assets?
    \item[] Answer: \answerNA{}
    \item[] Justification: The paper introduces no new datasets or pre-trained models. The only artefacts produced by SAP at run time are reasoning principles and routes, which are ephemeral test-time outputs rather than released assets.

\item {\bf Crowdsourcing and research with human subjects}
    \item[] Question: For crowdsourcing experiments and research with human subjects, does the paper include the full text of instructions given to participants and screenshots, if applicable, as well as details about compensation (if any)?
    \item[] Answer: \answerNA{}
    \item[] Justification: The paper does not involve crowdsourcing nor research with human subjects; all evaluations use established public benchmarks.

\item {\bf Institutional review board (IRB) approvals or equivalent for research with human subjects}
    \item[] Question: Does the paper describe potential risks incurred by study participants, whether such risks were disclosed to the subjects, and whether Institutional Review Board (IRB) approvals (or an equivalent approval/review based on the requirements of your country or institution) were obtained?
    \item[] Answer: \answerNA{}
    \item[] Justification: The paper does not involve crowdsourcing nor research with human subjects, so IRB review is not applicable.

\item {\bf Declaration of LLM usage}
    \item[] Question: Does the paper describe the usage of LLMs if it is an important, original, or non-standard component of the core methods in this research? Note that if the LLM is used only for writing, editing, or formatting purposes and does \emph{not} impact the core methodology, scientific rigor, or originality of the research, declaration is not required.
    %this research?
    \item[] Answer: \answerYes{}
    \item[] Justification: VLMs (Qwen3-VL-8B as the primary backbone, plus Qwen3-VL-2B/4B/30B, InternVL3.5-8B and DeepSeek-VL2 in ablations) are the core inference engines on which SAP operates and are explicitly described in Section~5 and Appendix~B. SAP uses them only via standard API-style inference calls; no fine-tuning is performed. LLM usage for paper writing/editing was limited to surface-level proofreading and is not a methodological component.

\end{enumerate}